\documentclass{article}
\usepackage{iclr2027_conference,times}

\usepackage{amsmath,amsfonts,bm}

\def\eqref#1{equation~\ref{#1}}

\def\1{\bm{1}}

\DeclareMathAlphabet{\mathsfit}{\encodingdefault}{\sfdefault}{m}{sl}
\SetMathAlphabet{\mathsfit}{bold}{\encodingdefault}{\sfdefault}{bx}{n}

\usepackage{hyperref}
\usepackage{url}
\usepackage{amssymb}
\usepackage{mathtools}
\usepackage{amsthm}
\usepackage{graphicx}
\usepackage{booktabs}
\usepackage{multirow}
\usepackage{makecell}
\usepackage{array}
\usepackage{enumitem}
\usepackage{nicefrac}
\usepackage{pifont}
\usepackage{microtype}
\usepackage{subcaption}
\usepackage{wrapfig}
\usepackage{adjustbox}
\usepackage{bbm}
\usepackage{tabularx}
\usepackage[table]{xcolor}
\usepackage{nicematrix}
\usepackage{afterpage}

\title{
Denoising Surface: Modeling and Predicting Inference Cost for Diffusion LLM Serving
}

\author{%
\textbf{Haoyu Zheng\textsuperscript{1}}
\quad
\textbf{Fangcheng Fu\textsuperscript{2}}
\quad
\textbf{Binhang Yuan\textsuperscript{3}}
\quad
\textbf{Yongqiang Zhang\textsuperscript{4}}
\quad
\textbf{Liang Deng\textsuperscript{4}}
\\[0.18em]
\textbf{Hao Wang\textsuperscript{1}}
\quad
\textbf{Yuanyuan Zhu\textsuperscript{1}}
\quad
\textbf{Xiao Yan\textsuperscript{1,*}}
\quad
\textbf{Jiawei Jiang\textsuperscript{1,5,*}}
\\[0.4em]
\textsuperscript{1}Wuhan University
\qquad
\textsuperscript{2}Shanghai Jiao Tong University
\\[0.05em]
\textsuperscript{3}The Hong Kong University of Science and Technology
\qquad
\textsuperscript{4}Damen Database Co., Ltd.
\\[0.05em]
\textsuperscript{5}Central China Normal University
\\[0.2em]
\textsuperscript{*}Corresponding authors
}

\iclrfinalcopy

\begin{document}

\newcommand{\stitle}[1]{\vspace{0em}\noindent{\bf #1\/}}

\newenvironment{myquote}
  {\begin{list}{}{\leftmargin=18pt\rightmargin=16pt}
   \item\relax\itshape}
  {\end{list}}

\newcommand{\squishlist}{
    \begin{list}{$\bullet$}
    {
        \setlength{\itemsep}{1pt}
        \setlength{\parsep}{1pt}
        \setlength{\topsep}{2.5pt}
        \setlength{\partopsep}{0.5pt}
        \setlength{\leftmargin}{1em}
        \setlength{\labelwidth}{1em}
        \setlength{\labelsep}{0.6em}
    }
}

\newcommand{\squishend}{
    \end{list}
}

\newcommand{\squishenumatelist}{
  \begin{enumerate}[
    itemsep=1pt,
    parsep=1pt,
    topsep=2.5pt,
    partopsep=0.5pt,
    leftmargin=1.5em,
    labelwidth=1em,
    labelsep=0.6em
  ]
}

\newcommand{\squishenumatend}{
  \end{enumerate}
}

\newtheorem{theorem}{Theorem}
\newtheorem{lemma}{Lemma}
\newtheorem{proposition}{Proposition}
\newtheorem{corollary}{Corollary}
\newtheorem{assumption}{Assumption}

\theoremstyle{remark}
\newtheorem{remark}{Remark}

\setlength{\fboxsep}{0.5pt}

\maketitle

\fancyhead{}
\renewcommand{\headrulewidth}{0pt}

        \vspace{-12pt}
\begin{abstract}
        \vspace{-6pt}

As diffusion large language models (dLLMs) become more capable, they are moving from research settings to real-world \textit{serving}, where request management (such as scheduling and resource allocation) relies on accurate estimation of per-request inference cost.
However, common cost proxies fall short for dLLMs: output length ignores that one forward pass can unmask multiple tokens, and denoising-step count ignores the \textit{heterogeneous} per-step costs.
We observe that the block-autoregressive generation mechanism induces a two-dimensional execution structure over output blocks and within-block denoising steps, whereas these proxies collapse it into a scalar, discarding information essential for characterizing the cost.
Motivated by this insight, we propose the Denoising Workload Surface (DWS), which preserves this two-dimensional block-step structure as a probability surface to weight the heterogeneous per-step costs.
We then design a coarse-to-fine training scheme that enables a lightweight prompt-only predictor to accurately predict the complex DWS.
This predictor runs efficiently even on a single CPU core, avoiding GPU contention with the serving model.
Since DWS decouples request-dependent execution behavior from deployment-specific cost factors, the predictor transfers across hardware configurations without retraining.
In \textit{real-world} serving experiments, DWS reduces cost-prediction error by up to $2.50\times$ over scalar-based predictors, while the DWS-guided shortest-job-first scheduler reduces end-to-end latency by up to $1.92\times$ for online chatbots.
\looseness=-1 %
\end{abstract}

        \vspace{-10pt}
\section{Introduction}
        \vspace{-6pt}
\label{sec:intro}

Diffusion Large Language Models (dLLMs)~\citep{nie2026large,gong2025scaling,li2026refusion} have attracted growing attention in recent years for their parallel decoding mechanism. 
Unlike autoregressive (AR) LLMs, which iteratively predict the next token and generate only \textit{one} token per decoding forward pass, dLLMs can jointly predict multiple masked positions and therefore unmask \textit{multiple} tokens within a single forward pass (i.e., a \textit{denoising step})~\citep{wu2026fast,wang2026diffusion}.
Importantly, recent dLLMs are competitive with small- and mid-scale autoregressive LLMs across various tasks~\citep{nie2026large,gong2025scaling,bie2025llada2}. 
Together with parallel decoding, these developments make dLLMs a promising alternative to AR LLMs for efficient serving~\citep{wang2026diffusion}.
Recently, popular LLM serving frameworks such as SGLang and vLLM have begun to support dLLM inference, signaling a shift from research settings to \textit{real-world serving}~\citep{sglang2025dllm,vllm2026diffusiongemma}.
\looseness=-1 %

Efficient LLM serving relies on effective request management, such as request scheduling, load balancing, and resource allocation.
For example, prior work on AR LLM serving commonly adopts shortest-job-first (SJF) scheduling to prevent long requests from blocking short ones, thereby reducing average latency and time-to-first-token (TTFT)~\citep{ssjf-aiops2024qiu,dont-stop-shahoutdon,zheng2026scheduling}.
Such request management relies on estimating how much computation each request will consume, making accurate cost estimation an important signal.
This raises our first question:
 
\vspace{-8pt}
\begin{myquote}
\textit{\textbf{Q1 (cost modeling)}: How can we accurately quantify the inference cost of a dLLM request?}
\vspace{-8pt}
\end{myquote}

\begin{figure}[t]
    \begin{subfigure}[b]{0.29\textwidth}
        \centering
        \includegraphics[width=\linewidth]{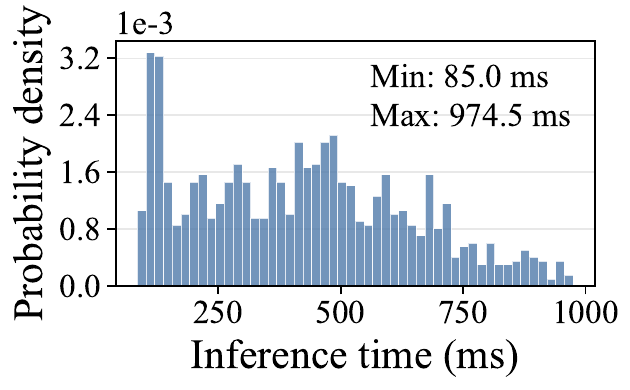}
        \vspace{-18pt}
        \caption{Requests with the same I/O lengths exhibit various costs.}
        \label{subfig:intro-a}
    \end{subfigure}
    \hfill
    \begin{subfigure}[b]{0.27\textwidth}
        \centering
        \includegraphics[width=\linewidth]{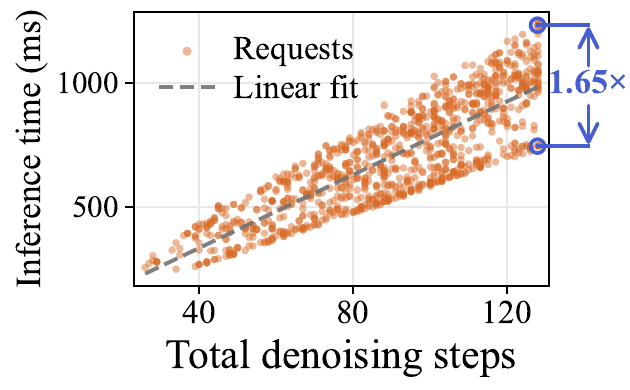}
        \vspace{-18pt}
        \caption{Denoising steps correlate with cost, but still inaccurate.}
        \label{subfig:intro-b}
    \end{subfigure}
    \hfill
    \begin{subfigure}[b]{0.4\textwidth}
        \centering
        \includegraphics[width=\linewidth]{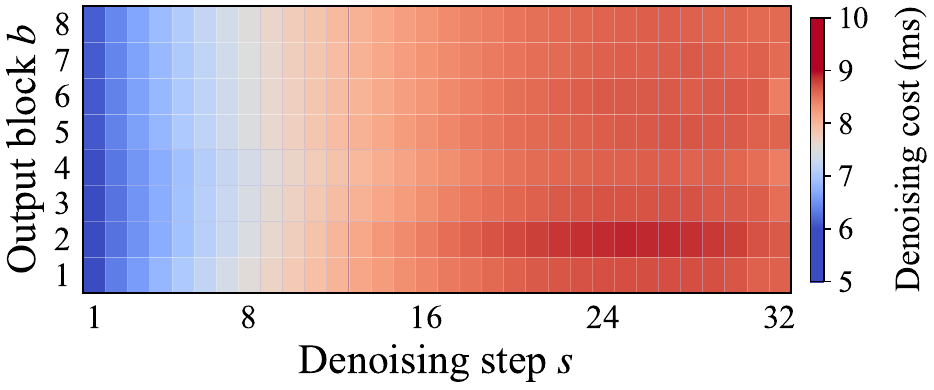}
        \vspace{-18pt}
        \caption{Cost of each denoising step varies across its positions, rather than being homogeneous.}
        \label{subfig:intro-c}
    \end{subfigure}
    
        \vspace{-6pt}
    \caption{\textbf{Motivation: common scalar proxies cannot faithfully characterize dLLM inference cost.} Setting: LLaDA2.0-mini serves \textit{equal-length} prompts \textit{individually}.}
        \vspace{-12pt}
    \label{fig:intro}
\end{figure}

\stitle{Output length?}
The question \textit{Q1} is well studied for AR LLMs, where each generated token corresponds to one decoding forward pass, making output length a natural proxy for request workload.
Accordingly, the SJF-based studies discussed above predict output length to estimate request cost.
However, this simple abstraction fundamentally \textbf{breaks down} for dLLMs, where a single denoising step can unmask a \textit{variable} number of tokens.
To empirically demonstrate this breakdown, we use LLaDA2.0-mini~\citep{bie2025llada2} and serve each request individually to avoid batching effects. 
As shown in Figure~\ref{subfig:intro-a}, even requests with the same input and output lengths exhibit substantially different (up to $11.47\times$) inference times, so output length is \textit{no longer a reliable proxy} for dLLM workload.
\looseness=-1 %

\stitle{Denoising steps?}
A natural alternative is the number of denoising steps, which directly counts the decoding forward passes. As shown in Figure~\ref{subfig:intro-b}, although it correlates with the inference time, requests with the same number of steps still vary considerably in inference cost. 
Figure~\ref{subfig:intro-c} reveals the reason: denoising steps at different positions can incur different computational costs\footnote{We find that this heterogeneity stems from multiple sources, including the model architecture and context length (detailed in Appendix~\ref{app:cost-heterogeneity}). The main text focuses on how to model and predict the resulting complex cost.}, especially across within-block step indices. However, simply counting the denoising steps treats \textit{heterogeneous} steps equally, making the total step count a reasonable but still \textit{inaccurate} proxy for dLLM workload.

\stitle{Our analysis and solution to Q1.}
Popular serving-oriented dLLMs typically generate in a block-autoregressive manner\footnote{We find that \textit{all} dLLMs currently natively supported by SGLang and vLLM follow this manner.}, where each output block undergoes multiple denoising steps (detailed in Section~\ref{sec:preliminaries}). This process naturally forms a \textbf{two-dimensional} execution structure indexed by the output block and the within-block denoising step. 
However, both output length and total denoising-step count compress this structure into a \textit{scalar}, discarding information essential for cost estimation. 
Motivated by this insight, we introduce the \textbf{D}enoising \textbf{W}orkload \textbf{S}urface (DWS), which preserves this structure by decomposing the denoising process into \textit{fine-grained} denoising cells. Then we weight the costs of each cell, thereby providing a more faithful characterization of dLLM inference cost.

\stitle{How to predict?}
A further challenge is to predict DWS before the request inference. 
A straightforward approach is to adapt existing AR-LLM workload predictors, which typically use a lightweight encoder (e.g., BERT) for prompt encoding, followed by a small regression head (e.g., an MLP) to predict the workload.
However, this design has two drawbacks for dLLM serving. 
\textit{First}, DWS is a complex two-dimensional target rather than a scalar, so a small regression head can struggle to learn it directly from the prompt.
\textit{Second}, although the encoder is lightweight compared with the serving model, it still requires GPU acceleration for efficient prediction, \textit{competing} with the serving model for GPU resources.
Moving it to CPUs can avoid GPU contention but incurs high prediction latency, delaying cost estimates and potentially leading to suboptimal serving decisions.
This motivates our second question:
\looseness=-1 %

\begin{myquote}
\vspace{-8pt}
\textit{\textbf{Q2 (cost prediction)}: How can we design a CPU-efficient predictor that enables accurate cost estimation for dLLM requests by capturing their structured inference workload?}
\vspace{-8pt}
\end{myquote}

\stitle{Our solution to Q2.}
We prove that DWS can be factorized along its two dimensions \textit{without} loss of expressiveness, and design a prompt-only predictor that separately predicts the corresponding distributions and combines them to \textit{reconstruct} the two-dimensional DWS.
Since learning the fine-grained DWS structure from prompts remains challenging, we further design a \textit{progressive} training scheme with \textit{multi-horizon} workload supervision to gradually guide the encoder from coarse workload patterns to the fine-grained DWS.
Finally, we carefully quantize the predictor to trade a \textit{negligible} loss in accuracy for substantially higher efficiency even deployed on a single CPU core.

Importantly, our formulation \textit{decouples} request-dependent execution behavior from deployment-specific cost factors. 
The former is captured by our accurate prompt-only DWS predictor, while the latter are automatically collected and refreshed through online profiling.
This separation allows the predictor to \textbf{transfer} across hardware deployments \textit{without} collecting new training data or retraining.

We evaluate both the predictor and its downstream serving performance. 
\textit{For the predictor}, DWS achieves up to $2.50\times$ lower cost MAE than common scalar-based predictors. 
Its quantization further improves the prediction throughput by $2.32\times$ with only $<1\%$ accuracy degradation.
\textit{For serving}, our DWS-guided SJF scheduler delivers substantial gains \textit{across} different scenarios, datasets, model scales, and hardware without retraining.
For online chatbots, it reduces end-to-end latency and TTFT by up to $1.92\times$ and $2.22\times$, respectively, over the default first-come-first-served (FCFS) scheduler. For offline data generation, it accelerates sample collection by up to $2.73\times$ over FCFS.
\looseness=-1 %

\stitle{Contributions.}
In summary, we make the following contributions:
{\setlength\topsep{0pt}\setlength\partopsep{1pt}
\begin{itemize}

    \item We conduct a fine-grained attribution of dLLM inference cost and reveal the heterogeneity of per-step costs. Based on these insights, we introduce DWS and a corresponding cost model that preserves the two-dimensional structure for \textit{accurate} estimation of inference cost.
    \looseness=-1 %
    \item We design a CPU-efficient prompt-only DWS predictor that accurately predicts the structured workload of dLLM inference. Moreover, our \textit{decoupled} design enables this predictor to transfer across deployments without collecting new training data or retraining.
    \item We show that DWS improves \textit{both} cost prediction accuracy and serving efficiency, enabling the DWS-guided scheduler to reduce online latency and accelerate offline data generation.
    \looseness=-1 %
\end{itemize}

        \vspace{-6pt}
\section{Preliminaries}
\label{sec:preliminaries}
        \vspace{-6pt}

In this section, we introduce the key mechanisms widely adopted by modern serving-oriented dLLMs.

\begin{wrapfigure}{r}{0.41\textwidth}
\vspace{-1.9em}
\centering
\includegraphics[width=\linewidth]{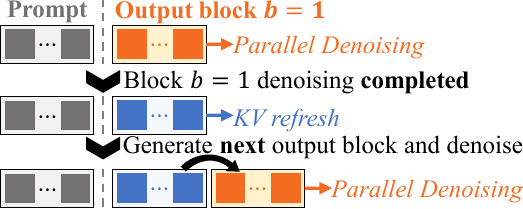}
\vspace{-1.8em}
\caption{Block-autoregressive generation with KV refresh.}
\label{fig:block_ar}
\vspace{-1.3em}
\end{wrapfigure}

\vspace{3pt}
\stitle{Block-autoregressive generation.}
Modern dLLMs commonly generate output blocks autoregressively~\citep{team2026diffusiongemma,bie2025llada2,song2025seed}, as shown in Figure~\ref{fig:block_ar}.
Within each output block, the model iteratively denoises masked tokens in parallel, where each denoising step performs one forward pass and reveals one or more tokens~\citep{arriola2025block}.
The block size $B$ is typically configured at deployment time and \emph{kept fixed} during inference, e.g., $B=32$ for LLaDA2.0.
Once a block is fully denoised, generation terminates if an \texttt{[EOS]} token has been produced; otherwise, a new masked block is appended and denoised.
Compared with early fixed-canvas diffusion, this design supports variable-length generation and avoids wasted computation on padded positions~\citep{arriola2025block}.
Thus, the generation workload substantially varies across requests, making cost-aware request scheduling beneficial.
\looseness=-1 %

\stitle{KV refresh.}
DLLMs can cache the key/value (KV) states of completed context to avoid redundant computation. However, bidirectional attention within a block keeps changing its KV states during denoising. Therefore, once a block is fully denoised, its \textit{own} KV states need to be refreshed and cached.
\looseness=-1 %

    \vspace{-6pt}
\section{Cost Modeling}
\label{sec:cost_modeling}
    \vspace{-6pt}

\subsection{Anatomy of dLLM Inference Cost}
\label{subsec:anatomy}
    \vspace{-6pt}

We begin with an empirical study of where dLLM inference time is actually spent.

\begin{wrapfigure}{r}{0.32\textwidth}
    \vspace{-18pt}
    \includegraphics[width=\linewidth]{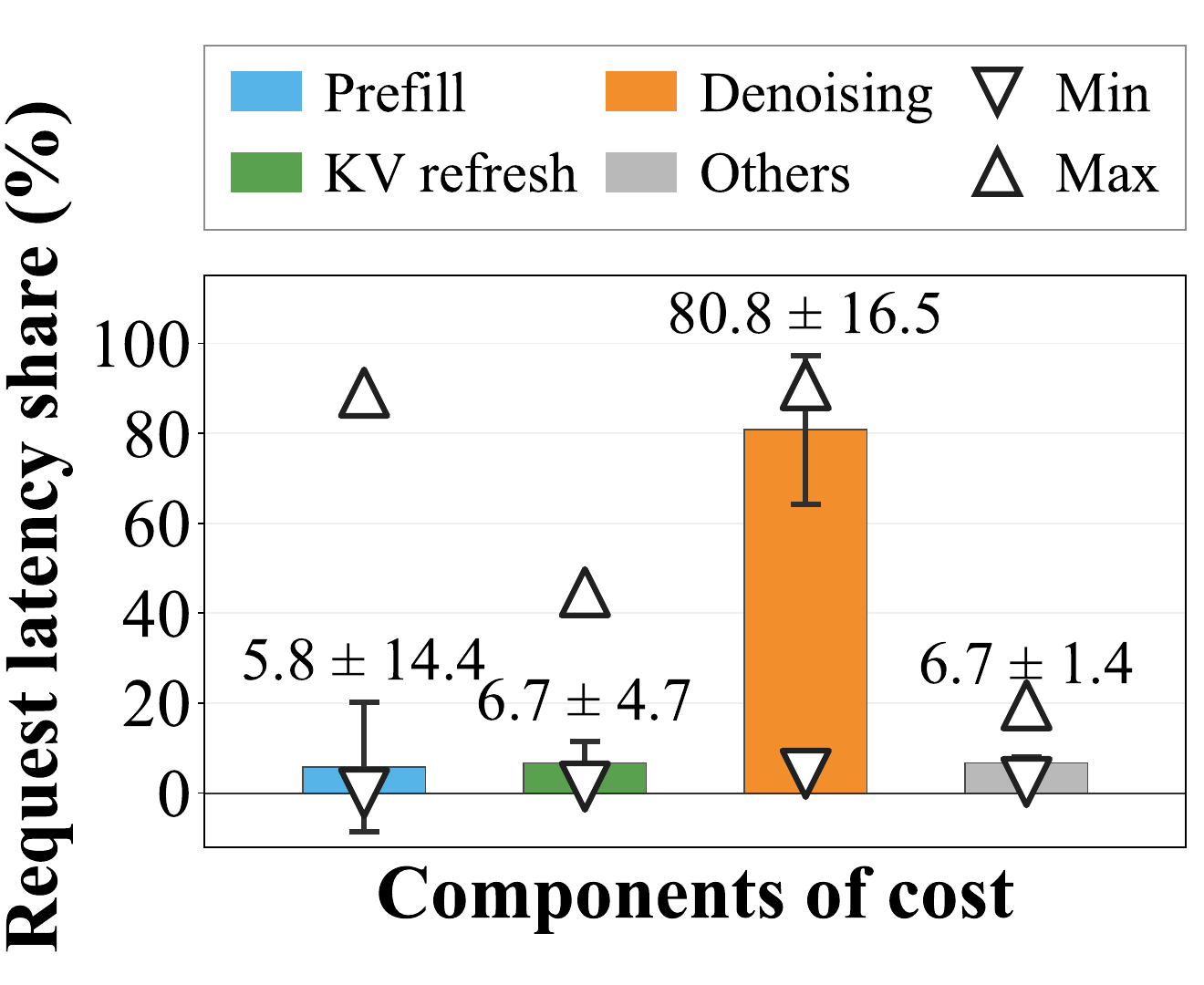}

    \vspace{-8pt}
    \caption{Per-request latency shares of inference stages (bars: mean $\pm$ std; markers: min/max).}
    \vspace{-18pt}
    \label{fig:component-breakdown}
\end{wrapfigure}
\stitle{Setting.}
We sample 20K requests from LMSYS-Chat-1M~\citep{zheng2023lmsys}, a real-world dataset.
Each is served by LLaDA2.0-mini \textit{individually} to isolate its \textit{own} cost from batching effects.
\looseness=-1 %

\stitle{Cost components.}
We analyze their latency breakdown to determine which stages to include in our cost model.
We retain a stage if it satisfies \textit{either} of two criteria: 
(i) it dominates the total latency; or (ii) it varies substantially across requests, and thus affects their relative ordering. 
Figure~\ref{fig:component-breakdown} summarizes the latency shares across these requests.
Overall, denoising dominates the total latency (avg: $80.8\%$), whereas prefill and KV refresh take up smaller average shares but vary substantially across requests ($5.8\pm14.4\%$ and $6.7\pm4.7\%$). 
In contrast, other stages are both small and stable, and thus have limited impact. 
Therefore, our cost model focuses on these three key components.
\looseness=-1 %

We first model the \textit{simpler} prefill and KV-refresh costs, deferring the \textit{complex} denoising to Section~\ref{subsec:denoising_cost}.

\stitle{Prefill and KV refresh.}
These two components follow simple execution patterns: prefill processes \textit{each} prompt block once, and KV refresh is performed once for \textit{each} newly completed output block.
These patterns naturally motivate \textit{linear} cost models in the prompt-block count $m$ and output-block count $n$, respectively.
As shown in Figure~\ref{fig:c_prefill_and_c_kv}, both linear models fit the measured costs \textit{closely} throughout the context range supported by LLaDA2.0 (up to 32K tokens, i.e., 1K blocks) with $R^2 \geq 0.999$ (detailed in Appendix~\ref{app:prefill-kv}).
We further study the higher-order terms introduced by attention computation under very long contexts in Appendix~\ref{app:long_context}.

\begin{wrapfigure}{r}{0.33\textwidth}
    \centering
        \vspace{-26pt}

    \begin{subfigure}[t]{\linewidth}
        \centering
        \includegraphics[width=\linewidth]{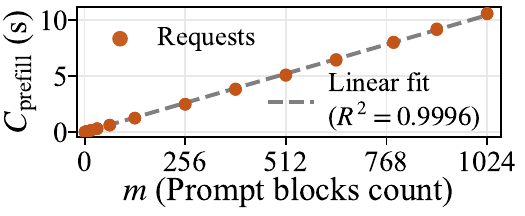}
        \vspace{-16pt}
        \caption{$C_{\mathrm{prefill}}$ scales linearly with $m$.}
        \label{subfig:c_prefill_vs_m}
    \end{subfigure}
    \begin{subfigure}[t]{\linewidth}
        \centering
        \vspace{6pt}
        \includegraphics[width=\linewidth]{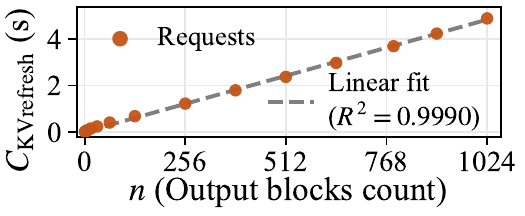}
        \vspace{-16pt}
        \caption{$C_{\mathrm{KVrefresh}}$ scales linearly with $n$.}
        \label{subfig:c_kv_vs_n}
    \end{subfigure}

    \vspace{-4pt}
    \caption{Linear fits for prefill and KV-refresh costs.}
    \label{fig:c_prefill_and_c_kv}
        \vspace{-26pt}
\end{wrapfigure}
We therefore model these two components as $C_{\text{prefill}} \triangleq \alpha m$ and $C_{\text{KVrefresh}} \triangleq \beta n$, so the total inference cost can be modeled as:
\begin{equation}
C \triangleq \alpha m + \beta n + C_{\text{denoising}}.
\label{eq:total-cost}
\end{equation}
The term $C_{\text{denoising}}$ dominates the total cost and is considerably harder to model, as we address below.

        \vspace{-3pt}
\subsection{Modeling the denoising cost}
        \vspace{-3pt}
\label{subsec:denoising_cost}

Recall that Figure~\ref{subfig:intro-c} shows that denoising steps at different execution positions differ in cost.
Thus, accurately modeling $C_{\mathrm{denoising}}$ requires preserving its \textit{fine-grained} denoising structure.
\looseness=-1 %

\stitle{A structured view of denoising.}
In dLLM inference, the number of output blocks $n$ is bounded by the configured maximum output length, so $n \in \{1,\cdots,N_{\max}\}$. 
Let $s_b$ denote the number of denoising steps required by the $b$-th output block $b$.
Since each step unmasks at least one token, $s_b \in \{1,\cdots,B\}$, where $B$ denotes the block size.
Writing $c_{b,s}^\star$ for the \textit{actual} cost of the denoising step $s$ in output block $b$, the \textit{actual} denoising cost of a request can be decomposed as
$C_{\mathrm{denoising}}^\star = \sum_{b=1}^{n} \sum_{s=1}^{s_b} c_{b,s}^\star$.

        \vspace{-3pt}
Here, although $n$ and $s_b$ vary across requests, their upper bounds ($N_{\max}$ and $B$, respectively) are fixed by the serving configuration. 
We therefore embed the denoising execution of each request into a fixed-size \textit{denoising activity matrix} $A \in\{0,1\}^{N_{\max}\times B}$, where $A_{b,s}=1$ iff the output \textit{reaches} block $b$ and performs \textit{at least} $s$ denoising steps within it, i.e., $A_{b,s}=\mathbbm{1}\{b \le n \land s \le s_b\}$.
This yields
    \vspace{-5pt}
\begin{equation}
C_{\mathrm{denoising}}^\star = \sum_{b=1}^{N_{\max}} \sum_{s=1}^{B} c_{b,s}^\star \cdot A_{b,s},
    \vspace{-2pt}
\label{eq:dws}
\end{equation}
which clearly separates these two \textit{orthogonal} factors: (i) the denoising activity carried by $A_{b,s}$, and (ii) the cost of each step carried by $c_{b,s}^\star$. 
Next, we turn to the cost of a single denoising step.

        \vspace{-2pt}
\stitle{Modeling cell costs.}
We first examine how the cost of the same denoising cell varies across requests.
Across the 20K traces detailed above, the 99th percentile (P99) coefficient of variation (CV) of $c_{b,s}^\star$ is only $11.08\%$, indicating that cell costs are fairly \textit{stable} across requests.
This suggests that a \textit{shared} cell-cost matrix is a reasonable starting point.
We further account for the residual variation induced by request-dependent factors.
Mechanically, a denoising step at block $b$ attends to the KV states of the preceding context, whose length is approximately $m+b-1$ blocks, where $m$ varies across requests.
We therefore decompose the effects of the execution position $(b,s)$ and the prompt length $m$:
\looseness=-1 %
    \vspace{-2pt}
\begin{equation}
K_{b,s}(m)
\triangleq
W_{b,s} + \delta (m-\bar m),
    \vspace{-2pt}
\end{equation}
where $K(m)\in\mathbb{R}^{N_{\max}\times B}$ denotes the modeled cost of all cells for a request with $m$ prompt blocks, and $W\in\mathbb{R}^{N_{\max}\times B}$ is the request-shared baseline cost matrix.
For the effects of context, the output-side context growth is already indexed by $b$ and thus absorbed into $W_{b,s}$, leaving only the prompt-side dependence, which is modeled by $\delta(m-\bar m)$, a request-dependent prompt-length correction.

We jointly calibrate $W$ and $\delta$ from request-level denoising latency (detailed in Appendix~\ref{app:profiling}).
Across 5-fold validation on the 20K real traces, this model achieves $R^2=0.9956$ and MAPE $=5.01\%$, showing an accurate approximation of denoising cost.
We further study its residual in Appendix~\ref{app:denoising_long_context}.

\stitle{Putting it together.} Substituting $K_{b,s}(m)$ into Equation~\ref{eq:dws}, the denoising cost can be modeled as:
    \vspace{-3pt}
\begin{equation}
C_{\mathrm{denoising}} \triangleq \sum_{b=1}^{N_{\max}} \sum_{s=1}^{B}
K_{b,s}(m) \cdot A_{b,s} = \langle K(m), A \rangle.
\label{eq:denoising_final}
        \vspace{-3pt}
\end{equation}
Moreover, the output blocks $n$ is determined by $A$ as the number of its active blocks, so $n=\sum_{b=1}^{N_{\max}}A_{b,1}$.
Finally, combining the prefill, KV-refresh, and denoising components, the total cost is:
        \vspace{-4pt}
\begin{equation}
C=\alpha m+\beta\sum_{b=1}^{N_{\max}}A_{b,1}+\langle K(m),A\rangle.
\label{eq:cost_final}
        \vspace{-3pt}
\end{equation}
Here, $m$ is directly observed from the prompt, while $\alpha$, $\beta$, and $K(\cdot)$ depend on the deployment configuration and can be obtained by profiling.
Next, we turn to the \textit{only} remaining unknown term $A$.
\looseness=-1 %

        \vspace{-6pt}
\section{Prediction of the denoising activity}
\label{sec:cost_prediction}
        \vspace{-8pt}

\subsection{Prediction Target}
\label{subsec:prediction_target}
        \vspace{-4pt}
Equation~\ref{eq:cost_final} \textit{naturally suggests} the prediction target, as derived below. First, the expected cost of $x$ is:
        \vspace{-6pt}
\begin{equation}
\mathbb{E}[C\mid x] = \alpha m(x) + \beta\sum_{b=1}^{N_{\max}}D_{b,1}(x) + \langle K(m(x)),D(x)\rangle,
\label{eq:expected_cost}
        \vspace{-8pt}
\end{equation}
where $D(x)\triangleq\mathbb{E}[A\mid x]\in[0,1]^{N_{\max}\times B}$ is the only unknown term in $\mathbb{E}[C\mid x]$, which is \textit{clearly decoupled} from the deployment-specific cost factors $\alpha$, $\beta$, and $K(\cdot)$. 
Next, we focus on $D(x)$.

Since $A_{b,s}$ is binary, $D_{b,s}(x)=P(A_{b,s}=1\mid x)$ is the activity probability of denoising cell $(b,s)$.
Geometrically, $D(x)$ forms a \textit{surface} over the block-step space, whose height at $(b,s)$ is $D_{b,s}(x)$, with an example shown on the right of Figure~\ref{fig:dws_predictor}.
We call $D(x)$ the \textbf{D}enoising \textbf{W}orkload \textbf{S}urface (\textbf{DWS}).
\looseness=-1 %

\stitle{Structure-preserving factorization.}
The decoding process imposes a \textit{survival structure}: $A_{b,s+1}\le A_{b,s}$ and $A_{b+1,1}\le A_{b,1}$ hold point-wise.
DWS inherits these constraints, and so \textit{should} its prediction.
However, directly predicting each cell independently may violate these constraints and produce structurally invalid activity patterns.
A straightforward fix is to explicitly enforce them on the predictor outputs, but the predictor itself still does not capture the underlying structure.

Instead, we build this decoding process into the parameterization, making the constraints hold \textit{by construction}.
Specifically, cell $(b,s)$ is active \textit{if and only if} decoding reaches block $b$ and this block runs at least $s$ steps, i.e., $\{A_{b,s}=1\}=\{n\ge b\}\cap\{s_b\ge s\}$.
By the chain rule, DWS factorizes exactly as
\looseness=-1 %
        \vspace{-2pt}
\begin{equation}
D_{b,s}(x) = P(A_{b,s}=1\mid x)
=
P(n\ge b\mid x)
\cdot
P(s_b\ge s\mid n\ge b,x),
\label{eq:dws_factorization}
        \vspace{-3pt}
\end{equation}
where both factors are survival probabilities and are thus non-increasing, satisfying these constraints \textit{naturally}.
We further prove that this factorization loses \textit{no} expressiveness (Appendix~\ref{app:factorized}).
Thus, DWS prediction decomposes into predicting two families of \textit{unconstrained} distributions, as described below.
\looseness=-1 %

\begin{figure}[t]
  \centering
    \centering
    \includegraphics[width=\textwidth]{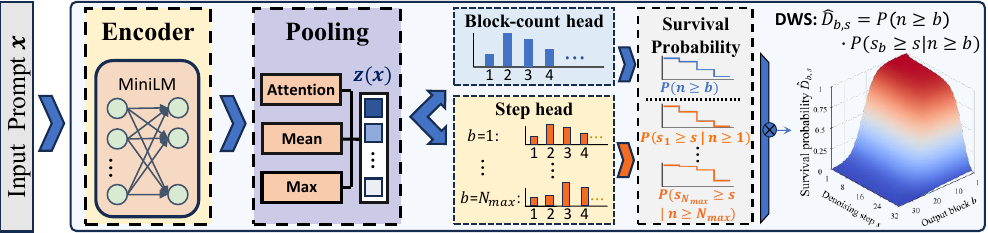}
        \vspace{-14pt}
    \caption{\textbf{Architecture of the DWS predictor.} An encoder extracts the representations of prompt $x$ and aggregates into $z(x)$, followed by two separate heads for the output-block and denoising-step distributions. Then the corresponding survival probabilities are combined to reconstruct the DWS $\widehat{D}$. DWS shown on the right is a smoothed visualization of $\widehat{D}$ for the first prompt in LMSYS-Chat-1M.}
        \vspace{-12pt}
    \label{fig:dws_predictor}
\end{figure}

        \vspace{-6pt}
\subsection{DWS Predictor}
        \vspace{-6pt}
\label{subsec:dws_predictor}

Following Equation~\ref{eq:dws_factorization}, we predict the two survival factors separately.
As shown in Figure~\ref{fig:dws_predictor}, given a prompt $x$, we use a pretrained all-MiniLM-L12-v2~\citep{wang2020minilm} encoder to encode its semantics and pool the token embeddings into a representation $\mathbf{z}(x)$.
On top of $\mathbf{z}(x)$, a block-count head predicts $p_k^{\mathrm{block}}(x) \triangleq P(n=k\mid x)$ for
$k\in\{1,\ldots,N_{\max}\}$, while a step head predicts
$p_{b,t}^{\mathrm{step}}(x) \triangleq P(s_b=t\mid n\ge b,x)$ for each potential
block $b$ and step $t\in\{1,\ldots,B\}$.
Given these distributions, we convert them into their corresponding survival probabilities and then reconstruct the predicted DWS as
\looseness=-1 %
        \vspace{-4pt}
\begin{equation}
\widehat D_{b,s}(x) =
\left(\sum_{k=b}^{N_{\max}} p_k^{\mathrm{block}}(x)\right) \cdot
\left(\sum_{t=s}^{B} p_{b,t}^{\mathrm{step}}(x)\right).
\label{eq:dws_predictor}
        \vspace{-10pt}
\end{equation}

The next question is how to train this predictor, facing two challenges: 
(i) the pretrained encoder is optimized for \textit{generic} semantic representations, whereas we ask it to represent the denoising process;
and (ii) DWS is a two-dimensional structured target and is inherently \textit{harder} to learn than a scalar workload.
We therefore adopt a \textit{two-stage} training scheme with \textit{multi-horizon} workload supervision.

\vspace{-3pt}
\stitle{Coarse-to-fine two-stage training.}
\textit{In Stage~A}, we introduce multi-horizon workload supervision to adapt the encoder.
Given a step horizon $h$, we truncate the denoising steps of every block at $h$, yielding the truncated workload $T^{(h)} \triangleq \sum_{b=1}^{n} \min(s_b, h) = \sum_{b=1}^{N_{max}} \sum_{s=1}^{h} A_{b,s}$. 
At $h=1$, $T^{(1)} = n$ focuses on the number of output blocks, while increasing $h$ \textit{progressively} exposes deeper within-block denoising behavior, until $T^{(B)}$ recovers the total denoising-step count.
To illustrate the benefits, consider two denoising traces with the same total step count: $\tau_1$ generates more blocks but requires fewer steps per block, whereas $\tau_2$ does the opposite. 
At a small $h$, $T^{(h)}(\tau_1) > T^{(h)}(\tau_2)$, since the deep denoising steps of $\tau_2$ are truncated.
As $h$ grows, these deeper steps are gradually exposed, and $T^{(h)}(\tau_2)$ approaches $T^{(h)}(\tau_1)$.
Supervising the encoder with $h \in \{1, 4, 8, 12, \ldots, B\}$ thus yields a multi-scale signal that a single scalar target cannot provide, teaching it to represent the \textit{coarse} denoising structure, which Stage~B then resolves into the fine-grained DWS.

\textit{In Stage~B}, we train the block-count head and the denoising-step head on top of the adapted encoder by directly supervising the reconstructed DWS, the final workload target, rather than the two heads separately.
Specifically, we minimize the Brier loss $\mathcal{L}_B = \tfrac{1}{N_{\max}B}\sum_{b,s}(\widehat D_{b,s}(x) - A_{b,s})^2$.
Minimizing this objective \textit{provably} controls the cost-estimation error under \textit{all} deployments simultaneously (detailed in Proposition~\ref{pro:error} in Section~\ref{sec:discussion}).
Further training details are provided in Appendix~\ref{app:predictor_training} and our code.
\looseness=-1 %

        \vspace{-2pt}
\stitle{Quantization.}
To run the predictor efficiently on CPUs, we quantize it to INT8, achieving $2.32\times$ speedup on a \textit{single CPU core}, with only a $0.54\%$ degradation in total-cost MAE (Appendix~\ref{app:predictor_quantization}).

        \vspace{-2pt}
\stitle{Pipeline.}
Upon request arrival, the CPU-side predictor takes its prompt $x$ and predicts its DWS $\widehat D(x)$, while $m(x)$ is obtained directly from the prompt length.
Equation~\ref{eq:expected_cost} then combines them with the profiled cost factors $\alpha$, $\beta$, and $K(\cdot)$, yielding the predicted inference cost $\widehat C(x)$.
\looseness=-1 %

        \vspace{-6pt}
\section{Discussion}
\label{sec:discussion}
        \vspace{-6pt}

\stitle{Why do dLLMs need DWS?}
We divide the answer into three aspects.
\textit{First}, DWS is the right prediction target. By Equation~\ref{eq:expected_cost}, given the deployment-specific cost factors, $D(x)$ and the observed $m(x)$ can model the expected cost of \textit{any} prompt $x$ under \textit{any} hardware deployment. 
In contrast, scalar targets are essentially \textit{lossy} compressions of DWS:
(i) output length keeps only $\mathbb{E}[n \mid x] = \sum_b D_{b,1}(x)$, losing the denoising structure in each block;
(ii) total steps keeps only $\sum_{b,s} D_{b,s}(x)$, losing the heterogeneous cost of each step;
(iii) direct cost relies on its training deployment and must be retrained once the cost factors change.
Appendix~\ref{app:proxy_bias} further theoretically quantifies their bias.
Moreover, a single DWS predictor can control the cost-estimation error for \textit{all} deployments (proved in Appendix~\ref{app:theory_proof}):
\looseness=-1 %
\begin{proposition}[Error control]
\label{pro:error}
For any predictor $\widehat D$ and deployment $\theta$, the predicted cost satisfies
        \vspace{-4pt}
\begin{equation}
\mathbb{E}_x\big|\widehat C_\theta(x)-\mathbb{E}[C_\theta\mid x]\big|\le\kappa_\theta\sqrt{N_{\max}B\,\big(\mathcal{R}(\widehat D)-\mathcal{R}(D)\big)},
\label{eq:uniform_bound}
        \vspace{-6pt}
\end{equation}
where $\mathcal{R} \triangleq \mathbb{E}[\mathcal{L}_B]$ is the Brier risk and $\kappa_\theta$ depends only on the profiled cost factors.
\end{proposition}
        \vspace{-6pt}
This bound depends on the predictor \textit{only} through $\mathcal{R}(\widehat D)$, which is why we minimize it in Stage~B.
\looseness=-1 %

\textit{Second}, beyond more accurate cost estimation, the \textit{structured} representation of DWS also supports serving decisions at different \textit{granularities}. 
(i) At the request level, it enables cost-aware scheduling such as SJF, which we evaluate in Section~\ref{subsec:application_evaluation}.
(ii) At the block level, DWS can predict the denoising steps within a block, thus supporting grouping requests with similar predicted next-block workloads to mitigate intra-batch stragglers, or estimating remaining workload for preemptive scheduling.
(iii) Moreover, the predicted output blocks can support KV-Cache estimation and memory allocation.

\textit{Third}, the decoupled design allows DWS to \textit{transfer} across deployments, which is crucial for serving, whereas a direct cost predictor must be retrained.
The transferring experiments in Appendix~\ref{app:cross_deployment} show that the unchanged DWS predictor can improve serving performance across deployments.

\stitle{Effect of individual profiling.}
We measure requests \textit{individually} to accurately isolate and attribute per-request costs. 
Batching and other runtime factors (such as KV-Cache reuse) may alter the cell costs, similarly to changes in hardware, but leave the DWS and its predictor \textit{unchanged}.
Moreover, using such request-level cost estimates to guide batch-level decisions is common practice in LLM serving~\citep{zheng2026scheduling,sheng2024fairness,fu2024efficient}. 
Empirically, with a batch size of 96 (the maximum before OOM), DWS-guided scheduling consistently performs best (shown in Section~\ref{subsec:application_evaluation}).
\looseness=-1 %

\stitle{Limitations.}
(i) For parsimony, our cost model omits small and stable stages (Section~\ref{subsec:anatomy}) and secondary effects such as the quadratic attention computation term.
We quantify the resulting residuals in Appendix~\ref{app:long_context} and find that they are well below the \textit{inevitable} prediction error.
(ii) Our decoupled design supports transfer across deployments, but assumes a \textit{fixed} model family (e.g., LLaDA2.0), block size, and decoding threshold, which typically serve as \textit{static} deployment configurations~\citep{bie2025llada2}.
Changing them requires retraining, which needs only activity matrices $A$ logged from routine serving.
\looseness=-1 %

        \vspace{-4pt}
\section{Experimental Evaluation}
\label{sec:experiment}
        \vspace{-8pt}
In this section, we evaluate (i) the \textit{fidelity} of DWS-based cost model, 
(ii) the \textit{accuracy} and CPU \textit{efficiency} of predictor, and 
(iii) the end-to-end \textit{serving performance} of DWS-guided SJF scheduler.

        \vspace{-6pt}
\subsection{Experimental Settings}
\label{sec:experimental_settings}
        \vspace{-6pt}

\stitle{Testbed and Models.}
Our experiments are conducted on a server with NVIDIA PRO 6000 GPUs (96\,GB each), AMD EPYC 9534 CPUs (64 physical cores each), and 256\,GB of host memory.
We use popular SGLang-0.5.10 as the serving framework and evaluate two advanced dLLMs, LLaDA2.0-mini (TP=1) and LLaDA2.0-flash~\citep{bie2025llada2} (TP=4).

\stitle{Serving configurations.}
Following prior work~\citep{chen2026tide}, we set the maximum output length to 1024 tokens.
The batch size is 96 (the maximum before OOM).
Unless otherwise specified, we use the standard decoding configuration of LLaDA2.0 (e.g., $B=32$, and decoding threshold is $0.95$).
\looseness=-1 %

\stitle{Datasets.}
For serving experiments, we use two \textit{real-world serving} datasets, LMSYS-Chat-1M~\citep{zheng2023lmsys} and ShareGPT~\citep{sharegpt52k}, together with a synthetic dataset Alpaca~\citep{alpaca}. 
We \textit{deliberately} adopt the \textbf{same datasets} from prior AR-LLM serving studies~\citep{zheng2026scheduling, fu2024efficient} to evaluate DWS-based dLLM serving performance.
Evaluations of the cost model and predictor use LMSYS-Chat-1M, with the training and test prompts kept disjoint.

    \vspace{-6pt}
\subsection{Cost model evaluation}
\label{subsec:cost_model_evaluation}
    \vspace{-6pt}

\stitle{Settings.}
We evaluate the fidelity of different cost models here.
To isolate modeling error from prediction error, all models take ground-truth statistics (e.g., output blocks). 
We compare against two scalar proxies: output length and total denoising steps.
Since neither accounts for prefill or KV-refresh costs, we additionally evaluate the denoising component of DWS model alone for direct comparison.
Scalar predictions are mapped to costs via a linear fit on the same 1K profiling traces as DWS.
We partition the 20K traces described in Section~\ref{subsec:anatomy} into four \textit{equal-sized} bins by measured cost and conduct \textit{5-fold cross-validation} within each bin, with results averaged across folds.

\begin{figure}[t]
    \centering

    \includegraphics[width=0.7\textwidth]
    {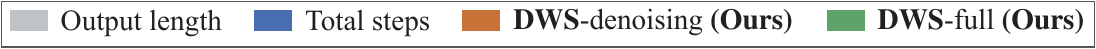}

    \begin{subfigure}[t]{0.267\textwidth}
        \centering
        \includegraphics[
            width=\linewidth,
            height=1.8cm,
            keepaspectratio
        ]{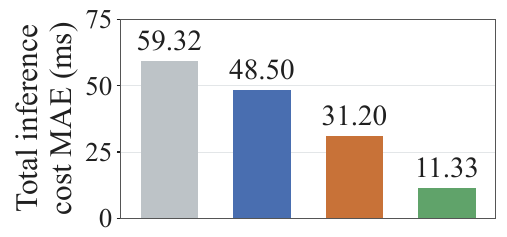}
    \vspace{-16pt}
        \caption{G1 (22-299 ms)}
    \end{subfigure}
    \hfill
    \begin{subfigure}[t]{0.235\textwidth}
        \centering
        \includegraphics[
            width=\linewidth,
            height=1.8cm,
            keepaspectratio
        ]{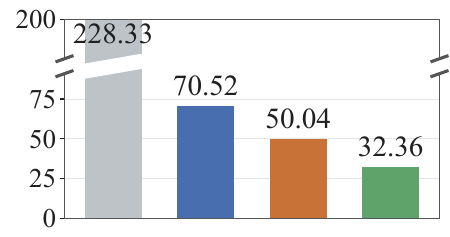}
    \vspace{-16pt}
        \caption{G2 (300-1136 ms)}
    \end{subfigure}
    \hfill
    \begin{subfigure}[t]{0.235\textwidth}
        \centering
        \includegraphics[
            width=\linewidth,
            height=1.8cm,
            keepaspectratio
        ]{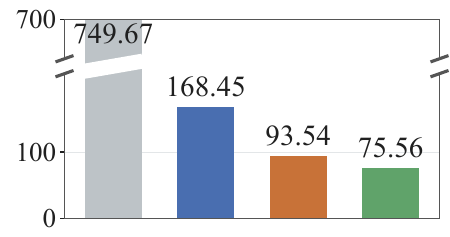}
    \vspace{-16pt}
        \caption{G3 (1137-3899 ms)}
    \end{subfigure}
    \hfill
    \begin{subfigure}[t]{0.235\textwidth}
        \centering
        \includegraphics[
            width=\linewidth,
            height=1.8cm,
            keepaspectratio
        ]{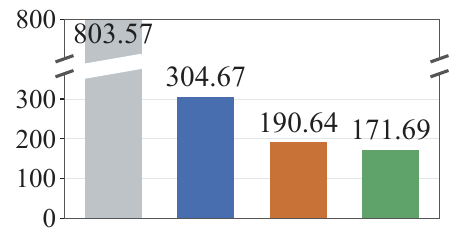}
    \vspace{-16pt}
        \caption{G4 (3900-8605 ms)}
    \end{subfigure}

        \vspace{-8pt}
    \caption{\textbf{Cost-model fidelity across request cost ranges.} DWS-denoising isolates the denoising component from our full model, enabling a direct comparison with scalar baselines. Lower is better.}
    \label{fig:cost_model_eval}
        \vspace{-13pt}
\end{figure}

\stitle{Results.}
As shown in Figure~\ref{fig:cost_model_eval}, Total Steps consistently outperforms Output Length, but still incurs clear error because it treats all denoising steps uniformly.
By accounting for where these steps occur, DWS-denoising reduces MAE by up to $1.80\times$ over Total Steps across all four cost ranges.
Further incorporating prefill and KV-refresh costs yields DWS-full, reducing the overall MAE by $2.04\times$ on average over Total Steps.
These results suggest that preserving the structured denoising process is essential, while incorporating prefill and KV-refresh costs further improves fidelity of the cost model.

    \vspace{-6pt}
\subsection{Predictor evaluation}
\label{subsec:predictor_evaluation}
    \vspace{-6pt}

\stitle{Settings.}
We evaluate the predictors in terms of cost-prediction accuracy, structure-prediction accuracy, and CPU efficiency, and further conduct ablation studies on the DWS predictor.
We train all predictors on the same 100K traces (10K for validation) and evaluate them on 10K disjoint test traces.
\looseness=-1 %

\stitle{Baselines.}
We compare DWS against three prompt-only predictors with different scalar prediction targets:
\textit{Output Length}, which predicts the generated output length;
\textit{Total Steps}, which predicts the total number of denoising steps;
and \textit{Direct Cost}, which directly regresses the measured per-request inference cost.
All of them use the \textit{same} MiniLM encoder.
We further ablate Stage A and multi-horizon supervision (i.e., use only $B$-horizon), and evaluate an INT8-quantized DWS predictor.

\begin{table}[t]
\centering
\caption{\textbf{Prediction accuracy of prompt-only predictors and \smash{\colorbox{cyan!7}{DWS \textit{ablations}}}.}
Cost MAE measures total-cost error; the others measure denoising-structure error. Throughput measurements use the same 256 prompts.
``--'' denotes a non-applicable metric. \textbf{Best} and \underline{second-best} results are marked.}
    \vspace{-6pt}
\label{tab:predictor_eval}
\resizebox{\columnwidth}{!}{%
\begin{NiceTabular}{lccccc}
\CodeBefore
\rowcolor{cyan!4}{6-8}
\Body
\toprule
\textbf{Predictor}
& \makecell[c]{\textbf{Cost}\\\textbf{MAE (ms)} $\downarrow$}
& \makecell[c]{\textbf{Total-step}\\\textbf{MAE} $\downarrow$}
& \makecell[c]{\textbf{Per-block step}\\\textbf{MAE} $\downarrow$}
& \makecell[c]{\textbf{Block-count}\\\textbf{MAE} $\downarrow$}
& \makecell[c]{\textbf{Throughput}\\\textbf{on CPU} (req/s) $\uparrow$} \\
\midrule
Output Length
& 887.97 & -- & -- & \textbf{2.01} & \underline{37.67} \\

Total Steps
& 517.04 & 62.50 & -- & -- & 37.54 \\

Direct Cost (\textit{untransferable})
& 432.29 & -- & -- & -- & 37.55 \\

DWS
& \textbf{355.24} & \textbf{56.06} & \textbf{4.41} & \underline{2.14} & 36.35 \\

\midrule
DWS (\textit{w/o} Stage A)
& 604.43 & 87.91 & 6.02 & 3.68 & 36.34 \\

DWS (\textit{single}-horizon)
& 445.65 & 59.70 & 5.33 & 2.82 & 36.51 \\

DWS (\textit{INT8})
& \underline{357.17} & \underline{56.42} & \underline{4.62} & 2.19 & \textbf{84.29} \\
\bottomrule
\end{NiceTabular}%
}
\end{table}

\stitle{Results.}
As shown in Table~\ref{tab:predictor_eval}, DWS achieves the lowest cost MAE, $1.22$--$2.50\times$ lower than the baselines.
Despite a similar total-step MAE to the total-step predictor, DWS reduces cost MAE by $1.46\times$, confirming the benefit of preserving the denoising \textit{structure} beyond merely counting steps.
Although Direct Cost performs better than other scalar baselines, it entangles prediction with deployment-specific costs and thus requires \textit{retraining} across deployments.

For ablations, removing Stage A or multi-horizon supervision can increase the cost MAE by $1.70\times$ and $1.25\times$, respectively, showing that both designs help the DWS predictor learn the denoising structure better.
INT8 quantization incurs only a \textit{negligible} accuracy loss ($<1\%$ in cost MAE), while improving throughput by $2.32\times$, enabling faster estimation of request inference cost to support downstream serving decisions and making the predictor more \textit{practical} for deployment.

    \vspace{-4pt}
\subsection{Real-world serving evaluation}
\label{subsec:application_evaluation}
    \vspace{-6pt}

\stitle{Settings.}
We evaluate shortest-job-first (SJF) scheduling with different predictors, where requests with lower predicted costs are given higher priority.
Following prior LLM serving systems~\citep{zheng2026scheduling,dont-stop-shahoutdon}, we incorporate a simple aging mechanism to \textit{prevent starvation}: a waiting request's predicted cost is discounted by 10\% every 30\,s until it is selected.

\stitle{Baselines.}
\textit{(i)~FCFS}, the \textit{default} scheduling policy in serving frameworks such as SGLang and vLLM, which can suffer from head-of-line (HOL) blocking~\citep{ssjf-aiops2024qiu}.
\textit{(ii)~PARS (GPU)}~\citep{tao2026ranking}, a recent SJF scheduler for AR LLMs, which builds on learning-to-rank~\citep{fu2024efficient} and introduces noise filtering, with a BERT-base-uncased~\citep{devlin2019bert} encoder. We adapt it to dLLM serving while preserving its \textit{output-length}-based ranking and its original deployment, where the predictor shares the \textit{GPU} with the serving model.
\textit{(iii)~PARS-Cost (GPU)} strengthens PARS by replacing its prediction target and ranking objective with the direct \textit{cost}.
\textit{(iv)~PARS-Cost (CPU)}, whose predictor runs on a single CPU core, \textit{matching} the resources of our predictor.
\textit{(v)~DWS-SJF}, which schedules requests using the cost estimated by our DWS cost model and DWS~(INT8) predictor.
\looseness=-1 %

\begin{figure}[t]
\centering

    \vspace{-6pt}
\includegraphics[width=\textwidth]{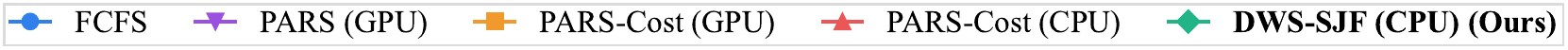}
    \vspace{-10pt}

\begin{subfigure}[t]{0.49\textwidth}
    \centering
    \includegraphics[width=0.49\linewidth]{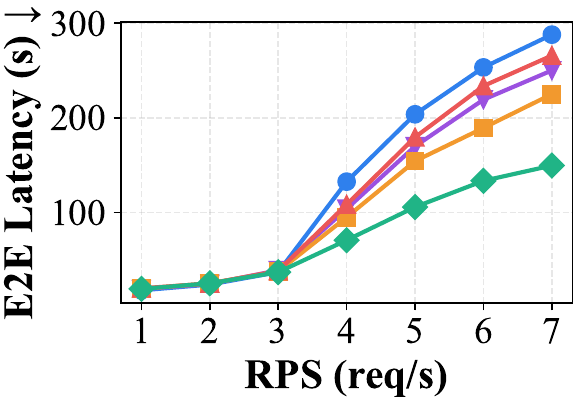}\hfill
    \includegraphics[width=0.49\linewidth]{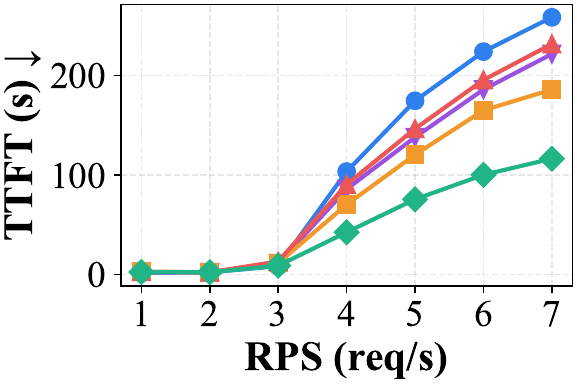}
    \vspace{-4pt}
    \caption{Performance under LMSYS-Chat-1M.}
    \label{fig:online_lmsys}
\end{subfigure}
\hfill
\begin{subfigure}[t]{0.49\textwidth}
    \centering
    \includegraphics[width=0.49\linewidth]{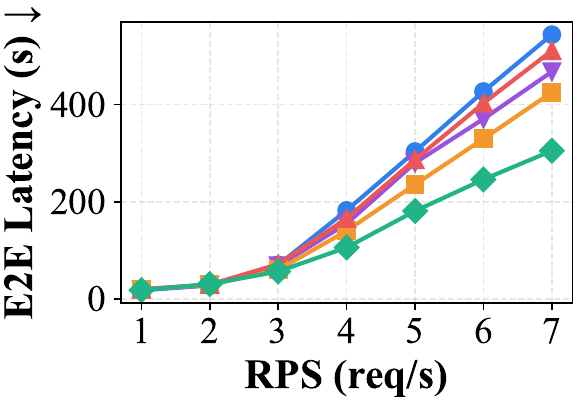}\hfill
    \includegraphics[width=0.49\linewidth]{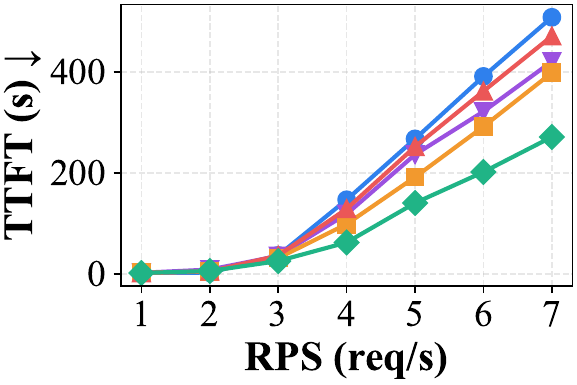}
    \vspace{-4pt}
    \caption{Performance under ShareGPT \textit{(cross dataset)}.}
    \label{fig:online_sharegpt}
\end{subfigure}

    \vspace{-4pt}
\caption{\textbf{Online chatbot serving performance.}
We report average E2E latency and TTFT (P90 in Appendix~\ref{app:additional_online}).
All predictors are trained on LMSYS-Chat-1M, so (b) tests cross-dataset \textit{generalization}.
    \vspace{-14pt}
\looseness=-1 %
}
\label{fig:online}

\end{figure}

\stitle{Online Chatbot.}
Following prior studies~\citep{zheng2026scheduling,fu2024efficient}, we vary the request rate (RPS) to evaluate each scheduler under different system loads, and report the average end-to-end (E2E) latency and TTFT. Additional metrics (e.g., P90 latency) and results are provided in Appendix~\ref{app:additional_online}.
\looseness=-1 %

\textit{Results.}
As shown in Figure~\ref{fig:online}, all schedulers perform similarly under light load, where requests rarely queue.
As the load increases, DWS-SJF grows the slowest and consistently achieves the lowest E2E latency and TTFT on both datasets.
On LMSYS-Chat-1M, it reduces E2E latency and TTFT by up to 1.92$\times$ and 2.22$\times$ over FCFS, and by up to 1.50$\times$ and 1.65$\times$ over PARS-Cost(GPU), the strongest baseline.
These gains carry over to ShareGPT \textit{without} retraining.
The results further reveal the significance of efficient prediction: PARS-Cost (CPU) falls behind because slow prediction delays cost estimates, thus the scheduler may blindly choose expensive requests, reintroducing the HOL blocking that SJF is meant to avoid.
This is precisely why we keep the DWS predictor efficient.

\begin{table}[t]
\centering
\caption{\textbf{Performance of schedulers on SDG tasks.}
All predictors are trained on LMSYS-Chat-1M traces collected with LLaDA2.0-mini. We evaluate \textit{(i)} in-domain, \smash{\colorbox{gray!16}{\textit{(ii)} cross-dataset}}, \smash{\colorbox{cyan!7}{\textit{(iii)} cross-model}}, and \smash{\colorbox{blue!7}{\textit{(iv)} cross-model-and-dataset}} settings to test the \textit{generalization} of these schedulers.}
    \vspace{-6pt}
\label{tab:sdg_performance}
\resizebox{\columnwidth}{!}{
\begin{NiceTabular}{lcllccccc}
\CodeBefore
\rectanglecolor{gray!15}{4-2}{5-9}
\rectanglecolor{cyan!6}{6-2}{7-9}
\rectanglecolor{blue!5}{8-2}{9-9}
\Body
\toprule
\textbf{\makecell{Serving\\Model}}
& 
& \textbf{Dataset}
& \textbf{SDG Metrics}
& \textbf{FCFS}
& \textbf{\makecell{PARS\\(GPU)}}
& \textbf{\makecell{PARS-Cost\\(GPU)}}
& \textbf{\makecell{PARS-Cost\\(CPU)}}
& \textbf{\makecell{DWS-SJF\\(CPU) (Ours)}} \\
\midrule

\multirow{4}{*}[-0.6ex]{\makecell[l]{LLaDA\\2.0-mini}}
& \multirow{2}{*}{\textbf{(i)}}
& \multirow{2}{*}{\makecell[l]{LMSYS-\\Chat-1M}}
& Time to 1K Req. (s) $\downarrow$
& 328 & 275 & \underline{168} & 290 & \textbf{120} \\
& & & Completed in 3 min $\uparrow$
& 517 & 661 & \underline{1021} & 625 & \textbf{1299} \\

\cmidrule(lr){2-9}

& \multirow{2}{*}{\textbf{(ii)}}
& \multirow{2}{*}{Alpaca}
& Time to 1K Req. (s) $\downarrow$
& 197 & 159 & \underline{108} & 165 & \textbf{75} \\
& & & Completed in 3 min $\uparrow$
& 912 & 1127 & \underline{1545} & 1069 & \textbf{1836} \\

\midrule

\multirow{4}{*}[-0.6ex]{\makecell[l]{LLaDA\\2.0-flash}}
& \multirow{2}{*}{\textbf{(iii)}}
& \multirow{2}{*}{\makecell[l]{LMSYS-\\Chat-1M}}
& Time to 1K Req. (s) $\downarrow$
& 1598 & 1194 & \underline{1085} & 1382 & \textbf{592} \\
& & & Completed in 3 min $\uparrow$
& 103 & 155 & \underline{181} & 138 & \textbf{310} \\

\cmidrule(lr){2-9}

& \multirow{2}{*}{\textbf{(iv)}}
& \multirow{2}{*}{Alpaca}
& Time to 1K Req. (s) $\downarrow$
& 909 & 752 & \underline{681} & 825 & \textbf{407} \\
& & & Completed in 3 min $\uparrow$
& 198 & 263 & \underline{292} & 245 & \textbf{468} \\

\bottomrule
\end{NiceTabular}
}
    \vspace{-14pt}
\end{table}

\stitle{Offline Synthetic Data Generation (SDG).}
SDG often aims to obtain a target number of samples from a large prompt pool as quickly as possible, making it a natural testbed for SJF scheduling. Following prior work~\citep{zheng2026scheduling,fu2024efficient}, we submit 3K requests at once and evaluate (i) how long to complete 1K requests and (ii) how many requests are completed within 3 minutes.
For cross-model evaluation, all cost-based schedulers are re-profiled on the \textit{same} 1K LLaDA2.0-flash traces.
\looseness=-1 %

\textit{Results.}
Table~\ref{tab:sdg_performance} shows that DWS-SJF performs best on both metrics across all serving models and datasets.
On LLaDA2.0-mini with LMSYS-Chat-1M, it completes the first 1K requests $2.73\times$ faster than FCFS and finishes $2.51\times$ more requests within 3 minutes, outperforming  PARS-Cost (GPU) (the strongest baseline) by $1.40\times$ and $1.27\times$, respectively.
These gains also hold on Alpaca and LLaDA2.0-flash \textit{without retraining}, suggesting its \textit{generalization} beyond the training workload and serving model.
Notably, the gains of PARS-Cost (GPU) shrink substantially when transferred to LLaDA2.0-flash, as its predictor is \textit{coupled} to deployment-specific costs.
In contrast, DWS-SJF retains most of its gains by \textit{decoupling} its predictor from deployment-specific cost factors.
\looseness=-1 %

\stitle{Additional studies.}
Due to space constraints, we defer the following studies to the appendix:
(i) additional results, including a MiniLM-based PARS-Cost baseline, a comparison with Oracle-SJF, and tail latency (Appendix~\ref{app:additional_online}),
(ii) cross-deployment experiments (PRO 6000 96G TP=1 $\rightarrow$ A6000 48G TP=2) (Appendix~\ref{app:cross_deployment}),
(iii) effect of the encoder (Appendix~\ref{app:encoder_effect}),
and (iv) long-context robustness (Appendix~\ref{app:long_context}).
\looseness=-1 %

    \vspace{-6pt}
\section{Related Work}
\label{sec:related_work}
    \vspace{-6pt}

Due to space constraints, we focus here on the most closely related dLLM serving studies and defer further discussion to Appendix~\ref{app:more_related_work}, including dLLM inference acceleration and AR LLM serving.

For dLLMs, a growing body of recent work focuses on improving the efficiency of dLLM serving.
dLLM-Serve~\citep{fan2026dllm} improves memory and resource efficiency through phase-aware execution, while Optimus~\citep{wei2026optimus} and BlockServe~\citep{zhu2026blockserve} improve serving efficiency under dynamic loads and heterogeneous convergence through adaptive decoding and batching.
DiLaServe~\citep{chang2026dilaserve} further targets SLO-aware serving with deadline-aware scheduling and adaptive resource management.
These systems primarily optimize dLLM execution at runtime, so they are \textit{complementary} to DWS: DWS provides a fine-grained and deployment-adaptable cost signal \textit{before} inference, which can further support scheduling and resource management.

    \vspace{-6pt}
\section{Conclusion}
\label{sec:conclusion}
    \vspace{-6pt}

We present DWS, a structured workload abstraction for dLLM serving.
Unlike scalar proxies such as output length or total denoising steps, DWS preserves the heterogeneous computation structure across output blocks and denoising steps, enabling a more \textit{faithful} estimation of request inference cost.
We further design an accurate prompt-only DWS predictor for each request's \textit{fine-grained} denoising structure, whose lightweight design enables \textit{efficient} prediction even on a single CPU core.
By \textit{decoupling} request-dependent execution behavior from deployment-specific cost factors, DWS can adapt across deployments through lightweight profiling without retraining its prompt-only predictor.
Experiments show that DWS consistently improves both cost-modeling fidelity and cost-prediction accuracy, yielding better serving performance with strong generalization.

\subsection*{AI use statement}

In this work, we used generative AI tools to improve the clarity and
readability of the manuscript.

\subsection*{Ethics statement}

We are not aware of any specific ethical concerns associated with this work.
Our experiments use publicly available datasets
(LMSYS-Chat-1M, ShareGPT, Alpaca) purely as serving workloads.

\subsection*{Reproducibility statement}

Implementation details and hyperparameters are provided in
Section~\ref{sec:experimental_settings} and
Appendix~\ref{app:predictor_training}.

\bibliography{references}
\bibliographystyle{iclr2027_conference}

\appendix

\section{Notation Conventions}
\label{app:notation}

\begin{table}[t]
\centering
\caption{Summary of notation.}
\label{tab:notation}
\setlength{\tabcolsep}{5pt}
\renewcommand{\arraystretch}{1.12}
\begin{tabularx}{\textwidth}{lX}
\toprule
\textbf{Symbol} & \textbf{Meaning} \\
\midrule

\multicolumn{2}{l}{\textit{Request and denoising execution}} \\[1pt]
$x$ &
Input prompt of a request. \\

$B$ &
Configured block size, i.e., the maximum number of tokens and denoising steps within one output block. \\

$N_{\max}$ &
Maximum number of output blocks allowed by the serving configuration. \\

$m$ &
Number of prompt blocks. \\

$n$ &
Number of generated output blocks. \\

$b$ &
Output-block index, $b\in\{1,\ldots,N_{\max}\}$. \\

$s$ &
Within-block denoising-step index, $s\in\{1,\ldots,B\}$. \\

$s_b$ &
Number of denoising steps executed for output block $b$. \\

$A\in\{0,1\}^{N_{\max}\times B}$ &
Denoising activity matrix of a realized request execution. \\

$A_{b,s}$ &
Activity indicator of cell $(b,s)$; $A_{b,s}=1$ iff block $b$ is generated and executes at least $s$ denoising steps. \\

\midrule
\multicolumn{2}{l}{\textit{Cost modeling}} \\[1pt]

$C$ &
Modeled total inference cost of a request. \\

$C_{\mathrm{prefill}}$,
$C_{\mathrm{KVrefresh}}$,
$C_{\mathrm{denoising}}$ &
Modeled costs of prefill, KV refresh, and denoising, respectively. \\

$c^\star_{b,s}$ &
Measured cost of denoising cell $(b,s)$ in a realized execution. \\

$K(m)\in\mathbb{R}^{N_{\max}\times B}$ &
Profiled cell-cost matrix for a request with $m$ prompt blocks. \\

$K_{b,s}(m)$ &
Modeled cost of denoising cell $(b,s)$ for prompt length $m$. \\

$W\in\mathbb{R}^{N_{\max}\times B}$ &
Request-shared baseline cell-cost matrix. \\

$\bar m$ &
Reference prompt-block count used when calibrating the cell-cost model. \\

$\delta$ &
Prompt-length correction coefficient in
$K_{b,s}(m)=W_{b,s}+\delta(m-\bar m)$. \\

$\alpha$ &
Profiled per-block coefficient for the prefill cost,
$C_{\mathrm{prefill}}(m)=\alpha m$. \\

$\beta$ &
Profiled per-block coefficient for the KV-refresh cost,
$C_{\mathrm{KVrefresh}}(n)=\beta n$. \\

\midrule
\multicolumn{2}{l}{\textit{DWS and workload prediction}} \\[1pt]

$D(x)\in[0,1]^{N_{\max}\times B}$ &
Denoising Workload Surface (DWS) conditioned on prompt $x$. \\

$D_{b,s}(x)$ &
Activity probability of cell $(b,s)$:
$P(A_{b,s}=1\mid x)$. \\

$\hat D(x)$ &
DWS predicted from prompt $x$. \\

$p^{\mathrm{block}}_k(x)$ &
Predicted probability $P(n=k\mid x)$ of generating exactly $k$ output blocks. \\

$p^{\mathrm{step}}_{b,t}(x)$ &
Predicted conditional probability
$P(s_b=t\mid n\ge b,x)$. \\

$z(x)$ &
Pooled prompt representation produced by the predictor encoder. \\

$h$ &
Step horizon used in multi-horizon workload supervision. \\

$T^{(h)}$ &
Workload truncated at horizon $h$:
$\sum_{b=1}^{n}\min(s_b,h)$. \\

\bottomrule
\end{tabularx}
\end{table}

Table~\ref{tab:notation} summarizes the main notation used throughout the paper.
We use $(b,s)$ to index a denoising cell, where $b$ denotes the output-block index and $s$ the within-block denoising-step index.
Quantities marked with $^\star$ denote measured execution costs, whereas their unmarked counterparts denote the corresponding quantities in our cost model.
For prediction, $D(x)$ denotes the prompt-conditioned DWS target and $\hat D(x)$ its predicted value.

\section{Modeling Prefill and KV-Refresh Costs}
\label{app:prefill-kv}

\begin{figure}[h]
\centering
    \includegraphics[width=0.6\textwidth]{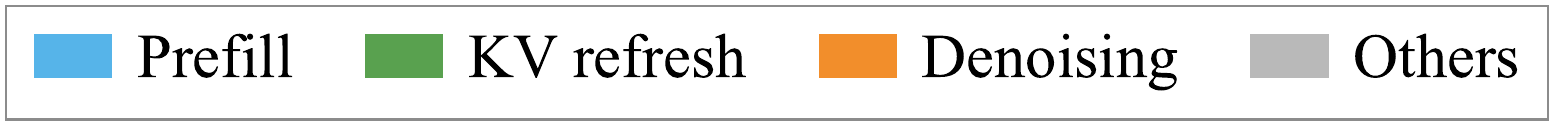}
    
\begin{subfigure}[t]{0.49\textwidth}
    \centering
    \includegraphics[
        width=\linewidth,
        height=3cm,
        keepaspectratio
    ]{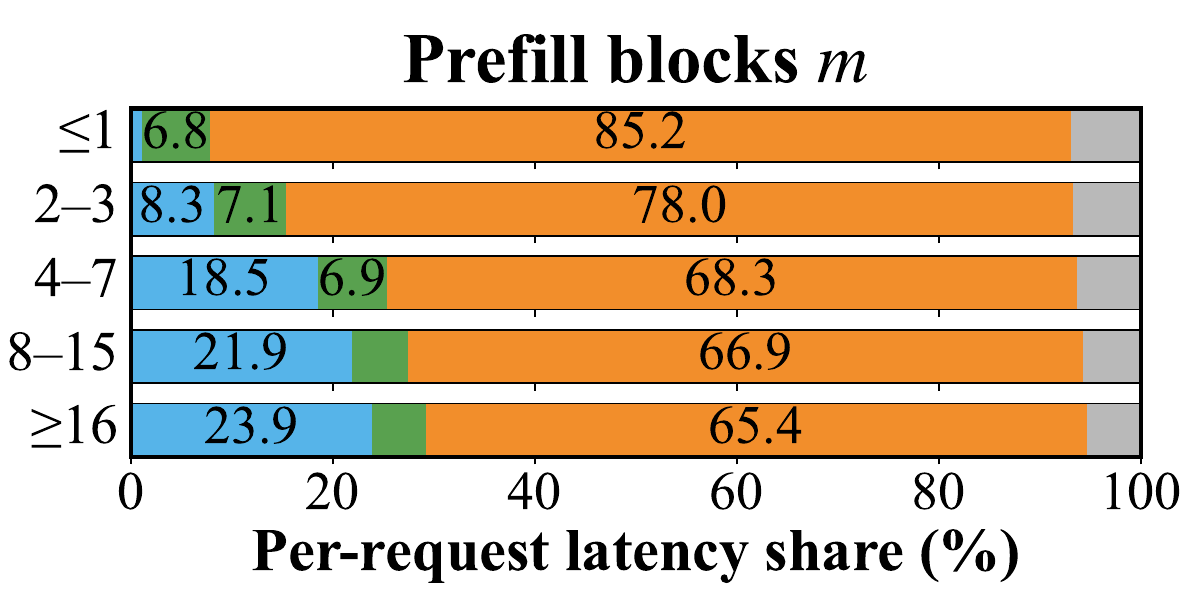}
    \caption{Varying prompt blocks.}
    \label{fig:component-share-vs-prompt}
\end{subfigure}
\hfill
\begin{subfigure}[t]{0.49\textwidth}
    \centering
    \includegraphics[
        width=\linewidth,
        height=3cm,
        keepaspectratio
    ]{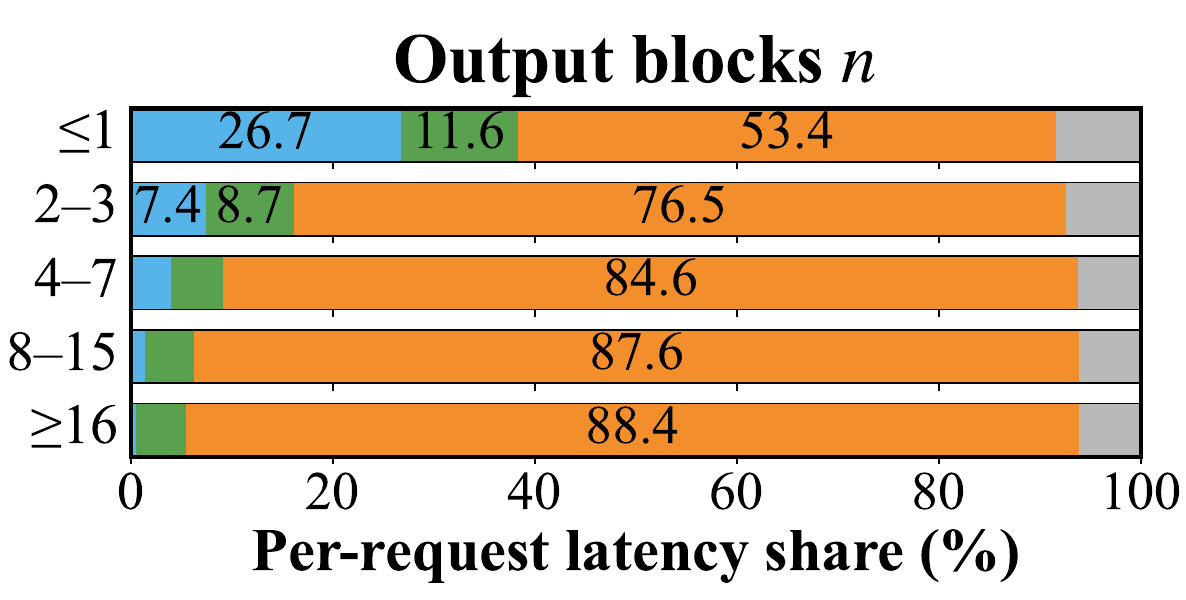}
    \caption{Varying output blocks.}
    \label{fig:component-share-vs-output}
\end{subfigure}
\caption{Latency breakdown under varying block counts.}
\label{fig:component-share-vs-blocks}

\end{figure}

\begin{table}[t]
    \centering
    \caption{Two-way ANOVA decomposition of component latency.}
    \label{tab:anova}

    \footnotesize
    \setlength{\tabcolsep}{3pt}
    \renewcommand{\arraystretch}{1.08}

    \begin{adjustbox}{max width=\linewidth,center}
        \begin{tabular}{@{}lccc@{}}
            \toprule
            \textbf{Component}
            & \textbf{Prompt blocks} $m$
            & \textbf{Output blocks} $n$
            & \textbf{Interaction} \\
            \midrule
            Prefill
            & \textbf{99.87\%}
            & 0.01\%
            & 0.12\% \\
            
            KV refresh
            & 0.15\%
            & \textbf{99.26\%}
            & 0.59\% \\
            \bottomrule
        \end{tabular}
    \end{adjustbox}
\end{table}

This section details the prefill and KV-refresh cost models summarized in Section~\ref{subsec:anatomy}.
We first identify their cost drivers (Appendix~\ref{app:cost-drivers}) and validate the linear models over the full supported context range (Appendix~\ref{app:linear-fit}), and then examine whether higher-order terms improve fidelity in long-context regimes (Appendix~\ref{app:long_context}).

\subsection{Cost Drivers}
\label{app:cost-drivers}
Mechanically, prefill processes each prompt block once, whereas KV refresh is performed once after each output block is fully denoised.
This suggests that the number of prompt blocks $m$ and the number of output blocks $n$ are their respective primary cost drivers.
To verify this, we apply a two-way ANOVA decomposition~\citep{fisher1930statistical,gelman2005analysis} to the 20K traces described in Section~\ref{subsec:anatomy}, quantifying the relative contributions of $m$, $n$, and their interaction.
As shown in Table~\ref{tab:anova}, $m$ accounts for 99.87\% of the decomposed variation in prefill latency, while $n$ accounts for 99.26\% of that in KV-refresh latency, with interaction contributions below 1\% in both cases.
These clear separations justify modeling prefill and KV-refresh costs as functions of $m$ and $n$, respectively.

\subsection{Linear Cost Models}
\label{app:linear-fit}
The ANOVA decomposition identifies the cost drivers but does not establish their functional relationship with cost.
We therefore profile each component while sweeping its driver over the full context range supported by LLaDA2.0.

\begin{figure}[h]
    \centering

    \begin{subfigure}[t]{0.45\linewidth}
        \centering
        \includegraphics[width=0.9\linewidth]{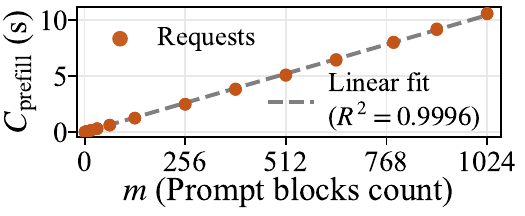}\
        \caption{$C_{\mathrm{prefill}}$ scales linearly with $m$.}
        \label{subfig:app_c_prefill_vs_m}
    \end{subfigure}
    \hfill
    \begin{subfigure}[t]{0.45\linewidth}
        \centering\
        \includegraphics[width=0.9\linewidth]{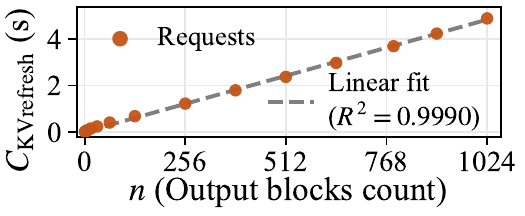}\
        \caption{$C_{\mathrm{KVrefresh}}$ scales linearly with $n$.}
        \label{subfig:app_c_kv_vs_n}
    \end{subfigure}

    \caption{Linear fits for prefill and KV-refresh costs.}
    \label{fig:app_c_prefill_and_c_kv}
\end{figure}

\stitle{Prefill cost.}
We sweep the prompt length from 32 tokens ($m=1$) to 32K tokens ($m=1{,}024$, the context limit of LLaDA2.0).
As shown in Figure~\ref{subfig:app_c_prefill_vs_m}, the linear model $C_{\text{prefill}}(m) \triangleq \alpha m$ fits the measurements closely ($R^2 = 0.9996$).
Although attention introduces an $O(m^2)$ term in principle, this term remains small relative to the linear cost of the QKV projections and feed-forward layers within this range, as further analyzed in Appendix~\ref{app:long_context}.

\stitle{KV-refresh cost.}
Since each completed output block triggers exactly one KV refresh, the number of refreshes equals $n$.
We sweep the output length from 32 tokens ($n=1$) to 32K tokens ($n=1{,}024$) by disabling early termination (\texttt{ignore\_eos=True}).
As shown in Figure~\ref{subfig:app_c_kv_vs_n}, the linear model $C_{\text{KVrefresh}}(n) \triangleq \beta n$ also fits the measurements closely ($R^2 = 0.9990$), with higher-order effects analyzed in Appendix~\ref{app:long_context}.

Figure~\ref{fig:component-share-vs-blocks} provides a complementary view by breaking down per-request latency by block counts.
Along the prompt dimension, the prefill share increases steadily with $m$, from 1.1\% for requests with at most one prompt block to 23.9\% for those with at least 16 blocks.
Along the output dimension, prefill and KV refresh together account for 38.3\% of latency for single-block outputs, while denoising increasingly dominates as $n$ grows, rising from 53.4\% to 88.4\%.
These trends are consistent with the ANOVA results and further show that prefill and KV refresh weigh most on short requests, which are exactly the requests that SJF scheduling must identify correctly.

\section{Analysis of the Per-Step Denoising Cost}
\label{app:cost-heterogeneity}

In the main text, we observe that denoising steps at different execution positions can incur different costs.
In this section, we further analyze the sources of such cost heterogeneity.

\textit{In short}, we identify two major sources of per-step denoising-cost heterogeneity: (i) dynamic expert routing in mixture-of-experts (MoE) architectures and (ii) the growing context length.
\textit{Notably}, the former is particularly relevant to practical dLLM serving: among all dLLM variants natively supported by the latest available releases of SGLang and vLLM\footnote{SGLang-0.5.20 (2026-09-18) and vllm-0.29.0 (2026-09-09)}, 9/10 use MoE backbones.

\subsection{Model Architecture}
\label{app:model-architecture}

We first analyze the pronounced within-block cost variation in LLaDA2.0-mini through a module-level latency breakdown.
Table~\ref{tab:moe-routing} compares the first and last denoising steps of 64 blocks.
The average forward latency increases from 6.043\,ms to 8.622\,ms.
This increase is concentrated in routed-expert computation: the total latency of the routed-expert layers increases from 3.038\,ms to 5.670\,ms, while the attention latency remains nearly unchanged.
Meanwhile, the average number of activated experts per layer increases from 32.7 to 66.3, and the average number of padded rows increases from 611 to 1,102.

\begin{table}[t]
\centering
\caption{Paired comparison between the first and last denoising steps within the same block on LLaDA2.0-mini.}
\label{tab:moe-routing}
\begin{tabular}{lrrr}
\toprule
Metric & First step & Last step & $\Delta$ \\
\midrule
Forward latency (ms) & 6.043 & 8.622 & +2.579 \\
Routed-expert layers, total (ms) & 3.038 & 5.670 & +2.632 \\
Attention layers, total (ms) & 1.114 & 1.116 & +0.002 \\
Activated experts per layer & 32.7 & 66.3 & +33.6 \\
Padded rows per layer & 611 & 1{,}102 & +491 \\
\bottomrule
\end{tabular}
\end{table}

We next examine the routing behavior underlying this change.
Although each token always selects a fixed number of experts, its hidden state evolves during denoising, which changes the token-to-expert assignments.
As these assignments become more dispersed, a block covers a larger set of experts.
In the current MoE kernel, token rows are padded separately for each expert; therefore, changes in the routing distribution can increase padded computation and alter expert-weight access patterns.

To further isolate the effect of routing itself, we fix the hidden states and assign identical weights to all experts, while changing only the expert-ID assignments.
Even with bitwise-identical model outputs, different routing patterns still incur different execution costs.
This result shows that the routing distribution itself can directly affect the cost of a denoising step.

\subsection{Context Length}
\label{app:context-cost}

Besides model architecture, the cost of a denoising step also depends on the preceding context length.
Within the same output block, the context length remains unchanged across denoising step $s$; instead, it grows with the input prompt length and as generation proceeds to later output blocks.

To quantify this effect, we vary the input context length while controlling the denoising state and block position.
As the input length increases from 256 tokens to approximately 32K tokens, the attention-kernel latency increases from 0.6091\,ms to 1.4104\,ms, while the overall forward latency increases from 6.5730\,ms to 7.3447\,ms.

Overall, the two factors affect complementary dimensions of denoising execution: dynamic routing induced by the model architecture mainly drives cost variation across $s$ within the same block, while context growth further changes the execution cost across requests and output-block positions.

\subsection{Implications for Serving Cost Modeling}
\label{app:heterogeneity-implication}

These analyses have different implications across model architectures.
(i) For dense dLLMs, within-block denoising costs can be relatively homogeneous, while context growth still introduces systematic cost variation across requests and output positions.
(ii) For the \textit{increasingly prevalent} MoE dLLMs, dynamic expert routing further introduces pronounced variation across within-block denoising steps, making the homogeneous per-step assumption used by total-step-based models increasingly inaccurate.

Importantly, DWS does not rely on MoE-specific mechanisms.
Its two-dimensional execution representation remains applicable to dense dLLMs, while its fine-grained modeling becomes particularly valuable for MoE models where per-step cost heterogeneity is stronger.
Specifically, the DWS-based cost model preserves the execution structure indexed by output block $b$ and within-block denoising step $s$, while separating request-dependent execution activity from deployment-specific cell costs.
The profiled cell-cost matrix captures position-dependent cost differences, including those arising from dynamic MoE execution, while prompt-side context dependence is modeled through the prompt-length correction and output-side context growth through the block-indexed cost profile.

Therefore, DWS provides a \textit{unified} structured formulation across dense and MoE dLLMs, while offering greater modeling benefits when the underlying per-step costs are more heterogeneous, as in prevalent MoE-based serving models.

\section{Residual Analysis of the Cost Model}
\label{app:long_context}
In this section, we analyze the residuals induced by the simplifications of our cost model.
We first quantify the effect of omitting the small and stable stages under real serving workloads (Appendix~\ref{app:omit}).
We then study how long prompts and long generated contexts affect the modeling fidelity of prefill, KV refresh, and denoising (Appendices~\ref{app:long_prompt_prefill}--\ref{app:denoising_long_context}).

\subsection{Effect of Omitted Stages}
\label{app:omit}

\stitle{Settings.}
Section~\ref{subsec:cost_model_evaluation} has shown the overall fidelity of our cost model in terms of MAE.
Here, we further examine whether the remaining modeling residuals affect the request ordering used by SJF.
We feed the ground-truth $m$ and $A$ from 3K LMSYS-Chat-1M traces into the cost model, thereby excluding prediction error.
The measured end-to-end cost includes the \textit{Others} stages omitted from our model, which account for only $6.7 \pm 1.4\%$ of latency (Figure~\ref{fig:component-breakdown}).
Therefore, the resulting residual reflects both these omitted stages and other modeling simplifications, such as the quadratic attention term omitted from our linear cost models.
We report both fitting accuracy and ranking consistency.

\stitle{Results.}
The modeled cost achieves a MAPE of $4.30\%$ and $R^2 = 0.9967$.
Importantly, it largely preserves request ordering, with a Spearman's $\rho$ of $0.9991$, and $99.21\%$ pairwise ordering accuracy.
Thus, the remaining modeling residuals have little impact on cost-aware scheduling.

\subsection{Prefill Cost under Long Prompts}
\label{app:long_prompt_prefill}

\stitle{Settings.}
As the prompt context grows, higher-order cost effects may become non-negligible; for example, attention computation can introduce a quadratic term in the prompt length.
Meanwhile, other major operations in prefill, such as projections and feed-forward layers, scale approximately linearly with sequence length.
Therefore, although the main paper adopts the parsimonious linear model $C_{\mathrm{prefill}}(m)=\alpha m$, we further examine whether incorporating such higher-order effects improves modeling fidelity under long prompts.
We compare it against a quadratic alternative, $C_{\mathrm{prefill}}(m)=\alpha_1m+\alpha_2m^2$, over prompt lengths up to 32K tokens, the maximum context length supported by LLaDA2.0.
Since real serving datasets contain relatively few prompts in these extreme length ranges, we construct prompts using randomly sampled tokens for this controlled experiment.
We report MAPE separately over six prompt-length ranges.

\begin{table}[t]
\centering
\caption{\textbf{Prefill-cost modeling under long prompts.}
MAPE of the linear model used in the main paper and a quadratic alternative across different prompt-length ranges.}
\label{tab:long_prompt_prefill}
\begin{tabular}{lcccccc}
\toprule
\textbf{Model} & $\leq$1K & 1--2K & 2--4K & 4--8K & 8--16K & 16--32K \\
\midrule
$\alpha m$ & 8.99\% & 5.01\% & 4.18\% & 3.03\% & 2.00\% & 1.61\% \\
$\alpha_1m+\alpha_2m^2$ & 7.42\% & 3.73\% & 2.50\% & 2.34\% & 1.47\% & 1.52\% \\
\bottomrule
\end{tabular}
\end{table}

\stitle{Results.}
As shown in Table~\ref{tab:long_prompt_prefill}, the prefill cost remains well approximated by the linear model throughout the supported context range.
Although the quadratic model is more flexible and consistently reduces fitting error, its additional gain becomes small in the long-context regime.
For 8--16K prompts, it reduces MAPE from $2.00\%$ to $1.47\%$, while for 16--32K prompts the improvement further shrinks to $1.61\%\rightarrow1.52\%$.
Thus, although self-attention theoretically introduces quadratic computation with sequence length, this higher-order effect provides little additional modeling fidelity within the context range supported by LLaDA2.0.
Moreover, under the real-world workloads characterized in Figure~\ref{fig:component-breakdown}, prefill accounts for only $5.8\%$ of the total inference latency on average, so such component-level improvements have an even smaller effect on the overall cost.
We therefore retain the linear form $C_{\mathrm{prefill}}(m)=\alpha m$ for parsimony.

\subsection{KV-Refresh Cost under Long Outputs}
\label{app:long_output_kv}

\stitle{Settings.}
Each completed output block triggers one KV-refresh operation.
While the refreshed block has a fixed size, later refreshes operate with an increasingly long generated context, so the cost of an individual refresh may grow with the output-block index.
For example, if the refresh cost at block $b$ grows approximately as $\gamma_0+\gamma_1 b$, accumulating it over $n$ output blocks gives
\begin{equation}
\sum_{b=1}^{n}(\gamma_0+\gamma_1 b)
=
\gamma_0 n+\frac{\gamma_1}{2}n(n+1),
\end{equation}
which motivates examining a quadratic component in the total KV-refresh cost.
We therefore compare the linear model $C_{\mathrm{KVrefresh}}(n)=\beta n$ used in the main paper against a quadratic alternative, $C_{\mathrm{KVrefresh}}(n)=\beta_1n+\beta_2n^2$.
For this experiment, we use real request traces to preserve the actual generation behavior of the dLLM, and raise the maximum generation length to 32K tokens, corresponding to the maximum context length supported by LLaDA2.0.
We report MAPE across six output-length ranges.

\begin{table}[t]
\centering
\caption{\textbf{KV-refresh-cost modeling under long outputs.}
MAPE of the linear model used in the main paper and a quadratic alternative across different output-length ranges.}
\label{tab:long_output_kv}
\begin{tabular}{lcccccc}
\toprule
\textbf{Model} & $\leq$1K & 1--2K & 2--4K & 4--8K & 8--16K & 16--32K \\
\midrule
$\beta n$ & 5.71\% & 6.05\% & 5.53\% & 6.60\% & 8.08\% & 10.17\% \\
$\beta_1n+\beta_2n^2$ & 4.15\% & 4.55\% & 4.82\% & 5.70\% & 7.68\% & 9.40\% \\
\bottomrule
\end{tabular}
\end{table}

\stitle{Results.}
As shown in Table~\ref{tab:long_output_kv}, the residual error of the linear KV-refresh model gradually increases for very long outputs, reaching $8.08\%$ for 8--16K tokens and $10.17\%$ for 16--32K tokens.
However, introducing the mechanism-motivated quadratic term provides only limited additional fidelity, reducing the corresponding errors to $7.68\%$ and $9.40\%$.
In particular, even in the longest 16--32K range, the quadratic model improves MAPE by only $0.77$ percentage points.
This indicates that the residual growth under very long outputs cannot be substantially explained by simply increasing the polynomial order of the model.
Moreover, under the real-world workloads characterized in Figure~\ref{fig:component-breakdown}, KV refresh accounts for only $6.7\%$ of the total inference latency on average, further limiting the impact of this modest component-level improvement on the overall cost.
Given the limited gain from the additional parameterization, we retain the linear form $C_{\mathrm{KVrefresh}}(n)=\beta n$ in the main cost model.

\subsection{Denoising Cost under Long Context}
\label{app:denoising_long_context}
We separately study how long prompts and long generated outputs affect the modeling fidelity of our denoising-cost model.
\stitle{Long prompts.}
As discussed in Section~\ref{subsec:denoising_cost}, the cost of a denoising step depends on the preceding context.
Our cost model captures the output-side context growth through the output-block index $b$ in the baseline cell-cost matrix $W$, while modeling the request-dependent prompt-length effect using a shared linear correction:
\begin{equation}
K_{b,s}(m)=W_{b,s}+\delta(m-\bar{m}).
\end{equation}
As the prompt becomes longer, however, higher-order effects from attention may become increasingly visible.
We therefore examine whether introducing an additional quadratic prompt-length term improves modeling fidelity:
\begin{equation}
K^{\mathrm{quad}}_{b,s}(m)
=
W_{b,s}
+\delta_1(m-\bar{m})
+\delta_2(m-\bar{m})^2.
\end{equation}
The corresponding request-level denoising cost is
\begin{equation}
C^{\mathrm{quad}}_{\mathrm{denoising}}
=
\langle W,A\rangle
+\delta_1(m-\bar{m})T
+\delta_2(m-\bar{m})^2T,
\end{equation}
where $T=\sum_{b,s}A_{b,s}$ denotes the total number of executed denoising cells.
Compared with the linear formulation, this alternative introduces only one additional scalar coefficient while retaining the same structured baseline matrix $W$.
Following the controlled setup in Appendix~\ref{app:long_prompt_prefill}, we construct prompts of different lengths using randomly sampled tokens, since real serving datasets contain relatively few prompts in the extreme length ranges.
The random-token construction is used to control the prompt length, while the denoising activity matrix $A$ and denoising latency are obtained from the actual generation process.
We sweep the prompt length up to 32K tokens, the maximum context length supported by LLaDA2.0, and report MAPE separately over six prompt-length ranges.

\begin{table}[t]
\centering
\caption{Denoising-cost modeling under long prompts. MAPE of the linear prompt-length correction and a quadratic alternative across different prompt-length ranges.}
\label{tab:denoising_long_prompt}
\small
\setlength{\tabcolsep}{5pt}
\begin{tabular}{lcccccc}
\toprule
Model & $\leq$1K & 1--2K & 2--4K & 4--8K & 8--16K & 16--32K \\
\midrule
$W_{b,s}+\delta(m-\bar{m})$
& \textbf{3.29\%} & \textbf{2.32\%} & 1.83\% & 2.03\% & 2.51\% & 3.57\% \\
$W_{b,s}+\delta_1(m-\bar{m})+\delta_2(m-\bar{m})^2$
& 3.57\% & 2.95\% & \textbf{1.53\%} & \textbf{1.49\%} & \textbf{1.87\%} & \textbf{2.08\%} \\
\bottomrule
\end{tabular}
\end{table}

As shown in Table~\ref{tab:denoising_long_prompt}, the benefit of the quadratic correction becomes increasingly apparent toward the long-context regime.
While the quadratic model slightly increases MAPE for prompts shorter than 2K tokens, it consistently improves over the linear model beyond 2K.
The reduction grows from 0.30 percentage points for 2--4K prompts to 0.54 and 0.64 points for 4--8K and 8--16K prompts, respectively, and reaches 1.49 points in the longest 16--32K range, reducing MAPE from 3.57\% to 2.08\%.
These results indicate that higher-order prompt-length effects become more pronounced under very long contexts.
At the same time, the linear correction remains accurate throughout the supported context range, with MAPE below 3.6\% in every range.
Given its low residual error and simpler parameterization, we retain the linear form $K_{b,s}(m)=W_{b,s}+\delta(m-\bar{m})$ for parsimony.

\stitle{Long outputs.}
We next examine the effect of long generated contexts on denoising cost.
Unlike the prompt-side dependence, output-side context growth is explicitly indexed by the output-block position $b$ and is thus largely captured by the corresponding entries of the baseline matrix $W$.
We therefore directly evaluate whether the existing cell-cost model remains accurate as generation proceeds to increasingly late output blocks, without introducing an additional output-length correction.
For this experiment, we use real request traces to preserve the actual denoising behavior of the dLLM and raise the maximum generation length to 32K tokens, corresponding to the maximum context length supported by LLaDA2.0.
We report MAPE separately over six output-length ranges.

\begin{table}[t]
\centering
\caption{Denoising-cost modeling under long outputs. MAPE of the profiled cell-cost model across different output-length ranges.}
\label{tab:denoising_long_output}
\small
\setlength{\tabcolsep}{6pt}
\begin{tabular}{lcccccc}
\toprule
\textbf{Output length} & $\leq$1K & 1--2K & 2--4K & 4--8K & 8--16K & 16--32K \\
\midrule
\textbf{MAPE} & 3.56\% & 3.14\% & 3.40\% & 4.86\% & 5.29\% & 5.73\% \\
\bottomrule
\end{tabular}
\end{table}

As shown in Table~\ref{tab:denoising_long_output}, the modeling error remains relatively stable for outputs up to 4K tokens, with MAPE between 3.14\% and 3.56\%.
As the generated context grows further, the residual error becomes increasingly visible: MAPE rises to 4.86\% for 4--8K outputs, 5.29\% for 8--16K, and 5.73\% in the longest 16--32K range.
This trend indicates that secondary context-dependent effects beyond the block-indexed baseline become more pronounced under very long generated contexts.
At the same time, the $b$-indexed baseline matrix $W$ remains accurate throughout the supported output range, with MAPE below 5.8\% in every range.
These results suggest that explicitly indexing the output-side context growth by $b$ captures its dominant effect without requiring an additional output-length correction.
Given the low residual error and the absence of an additional parameterization, we retain the original cell-cost formulation $K_{b,s}(m)=W_{b,s}+\delta(m-\bar{m})$.
\section{Automatic Profiling and Refresh of cost factors Online}
\label{app:profiling}

As discussed in Section~\ref{sec:cost_prediction}, DWS decouples the request-dependent execution structure from the deployment-specific cost factors $\alpha$, $\beta$, and $K(\cdot)$.
This separation allows the cost factors to be automatically calibrated from execution traces and periodically refreshed without retraining the DWS predictor.
We describe the profiling and fitting procedure below.

\stitle{Trace collection.}
For each profiled request $i$, the serving engine records its prompt-block count $m_i$, output-block count $n_i$, and the denoising depth $s_{i,b}$ of each generated block.
These observations directly determine the activity matrix $A_i$, where $A_{i,b,s}=\mathbbm{1}\{b\le n_i \wedge s\le s_{i,b}\}$, and the total number of denoising steps $T_i=\sum_{b,s}A_{i,b,s}$.
Meanwhile, CUDA-event instrumentation measures the corresponding prefill, KV-refresh, and denoising costs, denoted by $y_i^{\mathrm{pre}}$, $y_i^{\mathrm{kv}}$, and $y_i^{\mathrm{den}}$, respectively.
To isolate per-request cost from batching effects, profiling requests are measured individually under the target deployment configuration.

\stitle{Profiling prefill and KV-refresh costs.}
Recall that we model the prefill and KV-refresh costs as $\alpha m$ and $\beta n$, respectively.
Their coefficients can therefore be estimated directly by aggregating the corresponding measured costs:
\begin{equation}
\alpha =
\frac{\sum_i y_i^{\mathrm{pre}}}{\sum_i m_i},
\qquad
\beta =
\frac{\sum_i y_i^{\mathrm{kv}}}{\sum_i n_i}.
\label{eq:profile_alpha_beta}
\end{equation}
Both parameters can thus be refreshed using only simple running statistics.

\stitle{Profiling the cell-cost function.}
The denoising component is more challenging because it contains the fine-grained cell-cost function
$K_{b,s}(m)=W_{b,s}+\delta(m-\bar m)$.
For a profiling set of $N$ requests, we use the step-weighted prompt center
$\bar m=\frac{\sum_i m_iT_i}{\sum_i T_i}$,
such that the modeled denoising cost of request $i$ is
$\hat y_i^{\mathrm{den}}=\langle W,A_i\rangle+\delta(m_i-\bar m)T_i$.

We jointly fit the baseline cell-cost matrix $W$ and the prompt-length correction $\delta$ using weighted least squares:
\begin{equation}
\begin{aligned}
\min_{W\ge 0,\delta}\quad
& \frac{1}{N}\sum_i \omega_i
\left(
\langle W,A_i\rangle
+\delta(m_i-\bar m)T_i
-y_i^{\mathrm{den}}
\right)^2 \\
& + \lambda_s\|\nabla_s W\|_F^2
+ \lambda_b\|\nabla_b W\|_F^2 ,
\end{aligned}
\label{eq:cost_profile_fitting}
\end{equation}
where $\omega_i=(T_i+10)^{-2}$ reduces the dominance of long-running requests in the fitting objective.
The operators $\nabla_s$ and $\nabla_b$ denote first-order differences between neighboring cells along the denoising-step and output-block dimensions, respectively.
The corresponding smoothness terms stabilize the estimation under uneven cell coverage, while the non-negativity constraint on $W$ prevents invalid negative baseline cell costs.
We select $\lambda_s$ and $\lambda_b$ on held-out profiling traces and refit the parameters using the complete profiling set.

\stitle{Incremental statistics and refresh.}
Importantly, refreshing the profile does not require retaining and repeatedly processing all historical traces.
Let $a_i=\mathrm{vec}(A_i)$ and $z_i^0=m_iT_i$.
The data term in Equation~\ref{eq:cost_profile_fitting} can be represented by additive sufficient statistics, including
\begin{equation}
G=\sum_i\omega_i a_i a_i^\top,
\qquad
r=\sum_i\omega_i a_i y_i^{\mathrm{den}},
\qquad
q=\sum_i\omega_i a_i z_i^0,
\label{eq:profile_statistics}
\end{equation}
together with
$\sum_i\omega_i(z_i^0)^2$,
$\sum_i\omega_i z_i^0y_i^{\mathrm{den}}$,
$\sum_iT_i$, and
$\sum_i m_iT_i$.
Therefore, for a newly collected batch of traces, its statistics $\Delta\mathcal{S}$ can be merged into the existing profiling state $\mathcal{S}$ simply as
\begin{equation}
\mathcal{S}\leftarrow\mathcal{S}+\Delta\mathcal{S}.
\label{eq:profile_update}
\end{equation}
Once sufficient new profiling traces are accumulated, we recompute $\alpha$ and $\beta$ using Equation~\ref{eq:profile_alpha_beta} and re-solve Equation~\ref{eq:cost_profile_fitting} to refresh $W$ and $\delta$.
The resulting parameters constitute a new cost profile, while the DWS predictor itself remains unchanged.

\stitle{Deployment-specific profiles.}
Since the cost factors characterize the execution environment, each deployment configuration maintains its own cost profile.
Any change to the serving model, hardware, tensor-parallel configuration, or decoding configuration therefore starts a new profile with fresh statistics.

\stitle{Scope of predictor reuse.}
Although every such change requires a new cost profile, whether the DWS predictor must also be retrained depends on whether and how much the change alters the execution distribution $A \mid x$.
(i) Hardware deployments (e.g., changes in GPU type or tensor parallelism) leave $A \mid x$ unchanged, so re-profiling alone suffices, without collecting a new predictor-training dataset or retraining the predictor.
(ii) Changing the serving model within the same model family may alter $A \mid x$. Nevertheless, the experimental results in Section~\ref{subsec:application_evaluation} and Appendix~\ref{app:additional_online} show that the predictor remains effective with re-profiling alone (see Appendix~L.2 for the analysis).
(iii) In contrast, changing the core decoding configurations (e.g., block size) fundamentally reshapes the denoising structure captured by DWS, so the predictor needs to be retrained, as discussed in the limitations of Section~\ref{sec:discussion}.
In general, transfer across hardware deployments is the most common case, which our decoupled design handles without retraining the predictor.

\section{Predictor Quantization and CPU Efficiency}
\label{app:predictor_quantization}

\stitle{Motivation and setup.}
As discussed in Section~\ref{subsec:application_evaluation}, predictor efficiency can have a substantial impact on serving performance, since delayed cost estimates can weaken cost-aware scheduling decisions.
While GPU acceleration can reduce this overhead, deploying the predictor on GPUs would compete with the serving dLLM for compute and memory resources.
We therefore keep the DWS predictor on CPU and further explore reduced-precision inference to improve its efficiency without substantially degrading prediction accuracy.
Starting from the same trained DWS predictor, we evaluate four numerical formats, FP32, FP16, INT8, and INT4, under the same single CPU core setting.
We measure both prediction accuracy, including total-step MAE, per-block step MAE, and block-count MAE, and predictor throughput.

\begin{table}[t]
\centering
\caption{\textbf{Effect of predictor precision.}
All variants use the same trained DWS predictor and are evaluated on a single CPU core.
Lower is better for MAE, while higher is better for throughput.}
\label{tab:quantization}
\resizebox{\columnwidth}{!}{
\begin{tabular}{lcccc}
\toprule
\textbf{Precision}
& \textbf{Total-step MAE} $\downarrow$
& \textbf{Per-block step MAE} $\downarrow$
& \textbf{Block-count MAE} $\downarrow$
& \textbf{Throughput (req/s)} $\uparrow$ \\
\midrule
FP32 (\textit{Original}) & 56.06 & 4.41 & 2.14 & 36.35 \\
FP16 & 56.19 & 4.47 & 2.16 & 17.08 \\
INT8 & 56.42 & 4.62 & 2.19 & \textbf{84.29} \\
INT4 & 56.92 & 4.82 & 2.26 & 58.71 \\
\bottomrule
\end{tabular}
}
\end{table}

\stitle{Quantization results.}
As shown in Table~\ref{tab:quantization}, reduced-precision inference has only a minor impact on prediction accuracy.
This robustness is likely because our predictor primarily relies on \textit{coarse} prompt-level representations: token embeddings are first aggregated through pooling and then mapped to workload distributions by lightweight prediction heads, while the subsequent probability aggregation further reduces sensitivity to small quantization perturbations.
In particular, INT8 increases the total-step MAE from $56.06$ to $56.42$ ($0.64\%$) and the final cost MAE by only $0.54\%$ compared with FP32, while improving single-core throughput from $36.35$ to $84.29$ requests/s ($2.32\times$).
Given this \textit{negligible} accuracy degradation, we adopt INT8, which achieves the highest measured throughput, for all downstream serving experiments.

Interestingly, predictor throughput does not improve monotonically with lower numerical precision.
This is consistent with the hardware- and kernel-dependent nature of reduced-precision CPU inference: our Zen~4 CPU provides optimized integer execution (e.g., AVX-512 VNNI), while FP16 and INT4 may rely on less efficient kernels or incur additional conversion/packing overhead.
These architectural characteristics are consistent with INT8 achieving the highest measured throughput on our deployment.

\section{Characteristics of LLM Serving Workloads}
\label{app:workload_characteristics}

\begin{figure}[t]
    \centering

    \begin{subfigure}[t]{0.32\linewidth}
        \centering
        \includegraphics[width=\linewidth]{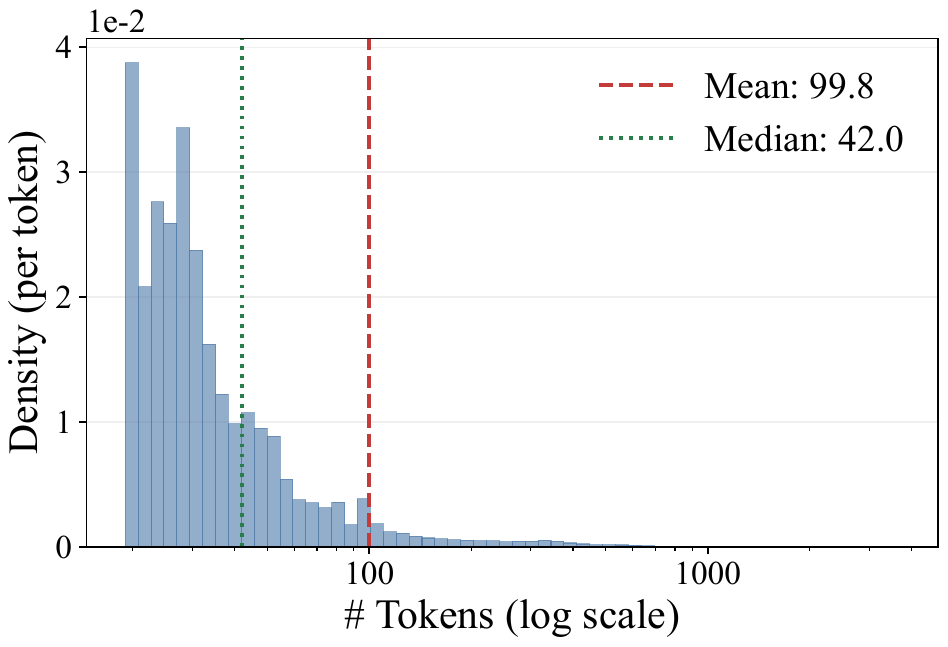}
        \caption{LMSYS prompt length.}
        \label{subfig:lmsys_prompt_length}
    \end{subfigure}
    \hfill
    \begin{subfigure}[t]{0.32\linewidth}
        \centering
        \includegraphics[width=\linewidth]{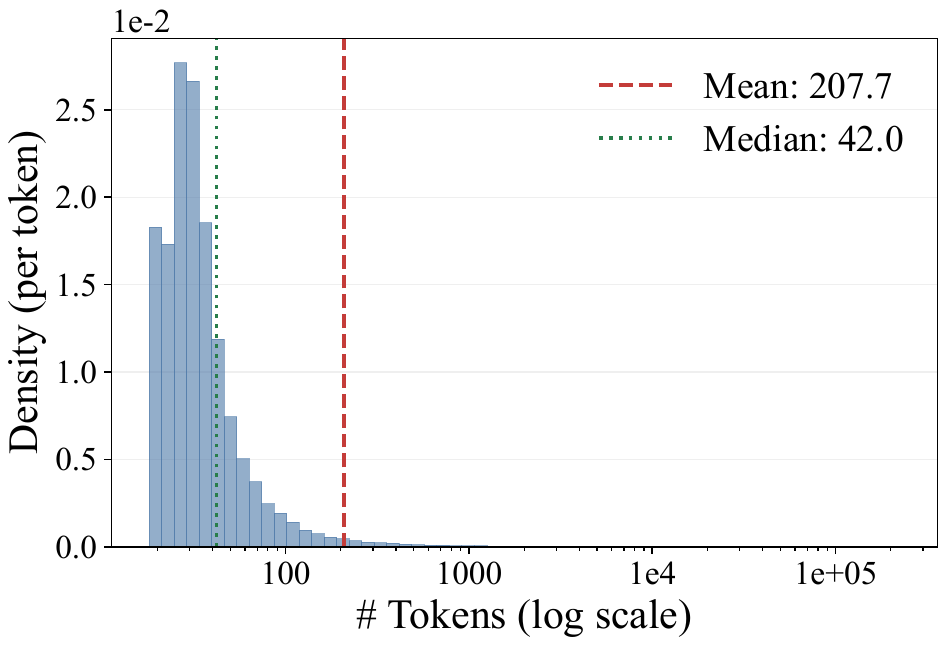}
        \caption{ShareGPT prompt length.}
        \label{subfig:sharegpt_prompt_length}
    \end{subfigure}
    \hfill
    \begin{subfigure}[t]{0.32\linewidth}
        \centering
        \includegraphics[width=\linewidth]{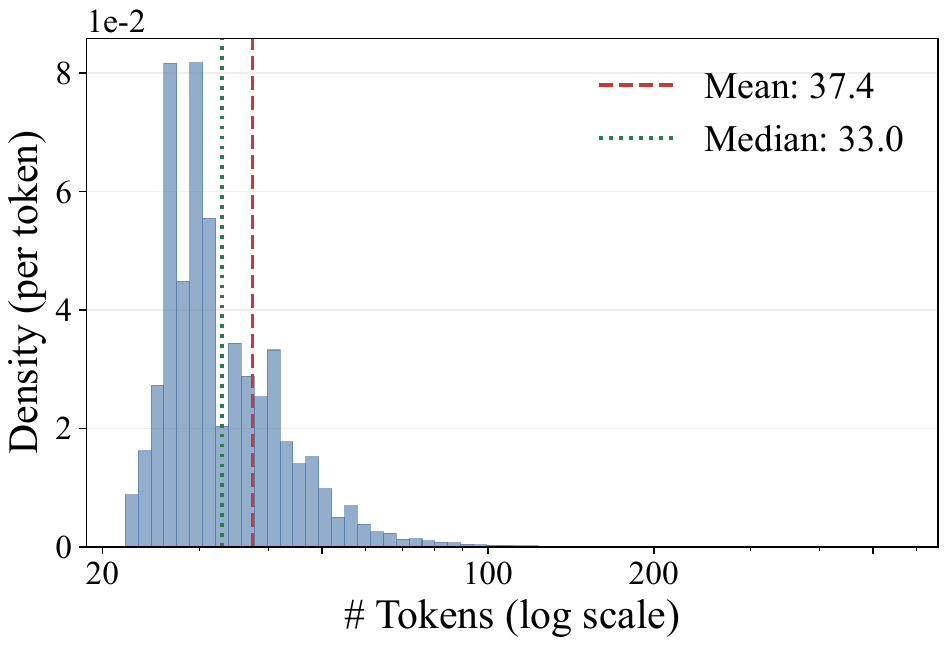}
        \caption{Alpaca prompt length.}
        \label{subfig:alpaca_prompt_length}
    \end{subfigure}

    \vspace{0.5em}

    \begin{subfigure}[t]{0.32\linewidth}
        \centering
        \includegraphics[width=\linewidth]{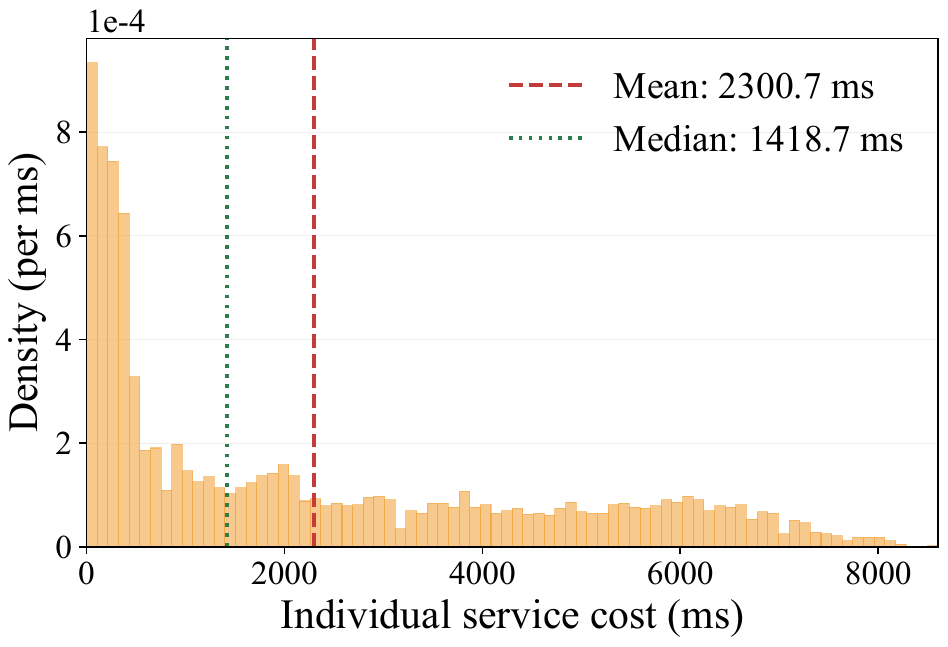}
        \caption{LMSYS inference cost.}
        \label{subfig:lmsys_cost}
    \end{subfigure}
    \hfill
    \begin{subfigure}[t]{0.32\linewidth}
        \centering
        \includegraphics[width=\linewidth]{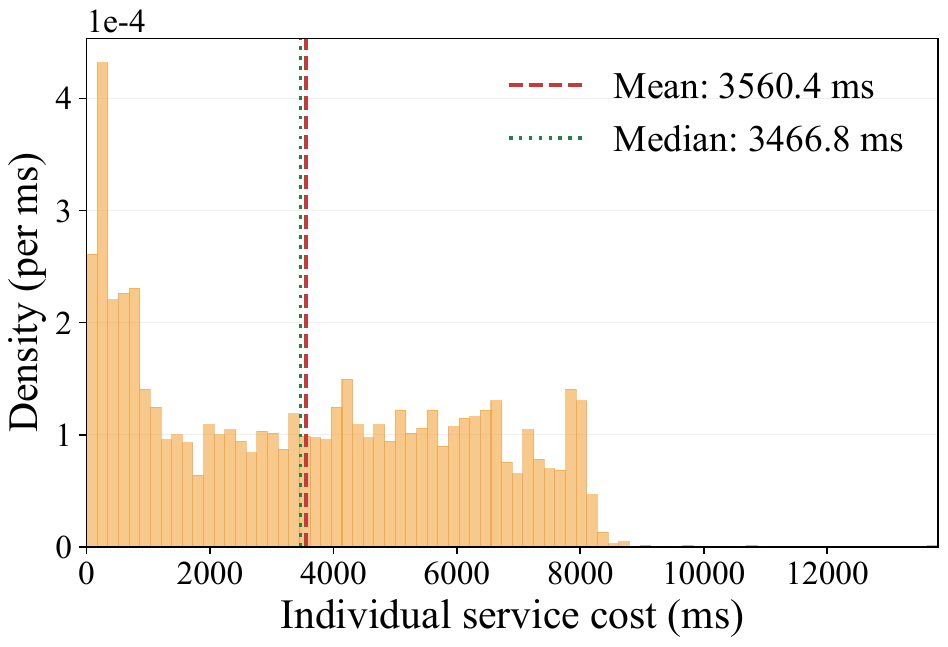}
        \caption{ShareGPT inference cost.}
        \label{subfig:sharegpt_cost}
    \end{subfigure}
    \hfill
    \begin{subfigure}[t]{0.32\linewidth}
        \centering
        \includegraphics[width=\linewidth]{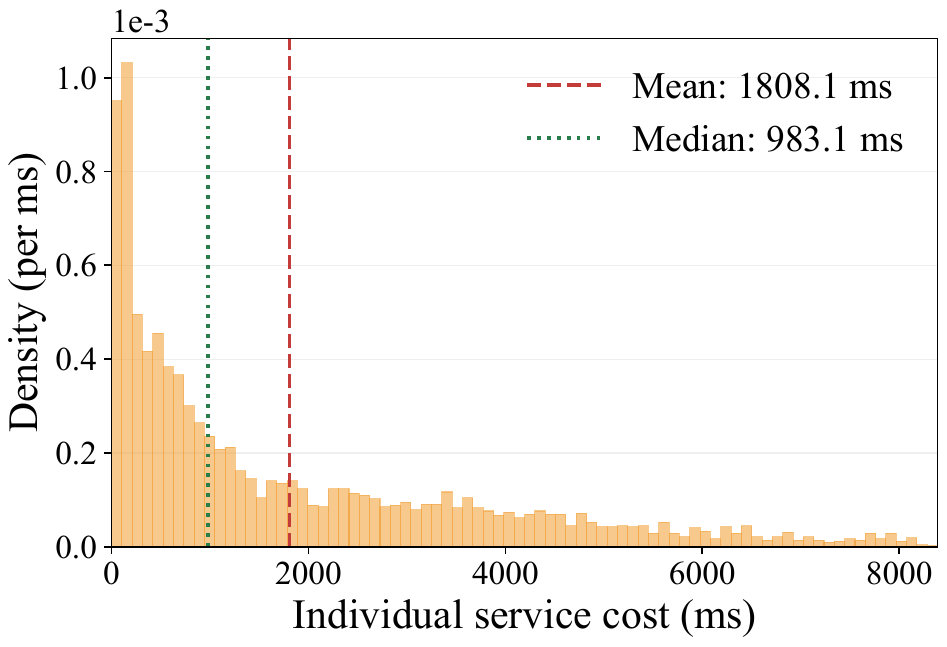}
        \caption{Alpaca inference cost.}
        \label{subfig:alpaca_cost}
    \end{subfigure}

    \caption{\textbf{Characteristics of LLM serving workloads.}
    Top: prompt-length distributions of LMSYS-Chat-1M, ShareGPT, and Alpaca.
    Bottom: individual service-cost distributions of 4K requests from each dataset served by LLaDA2.0-mini.
    Dashed and dotted lines denote the mean and median, respectively.}
    \label{fig:workload_characteristics}
\end{figure}

We further characterize the workloads used in our serving evaluation from two complementary perspectives: the input prompt length and the actual per-request inference cost.
For prompt length, we report the distributions over the complete LMSYS-Chat-1M, ShareGPT, and Alpaca datasets.
For inference cost, we additionally profile 4K individually served requests from each dataset with LLaDA2.0-mini, using the same serving configuration as in the main experiments.
Figure~\ref{fig:workload_characteristics} summarizes the resulting distributions.

\stitle{Prompt-length heterogeneity.}
The prompt distributions vary substantially across datasets.
LMSYS-Chat-1M has a mean/median prompt length of $99.8/42$ tokens, while ShareGPT has the same median of $42$ tokens but a substantially larger mean of $207.7$, reflecting its heavier long-context tail.
In contrast, Alpaca is considerably more concentrated around short prompts, with a mean/median length of $37.4/33$ tokens.
These differences expose substantial distribution shifts in the input characteristics encountered by the serving system.

\stitle{Inference-cost heterogeneity.}
The resulting dLLM inference costs are also highly heterogeneous, both within and across workloads.
For LMSYS-Chat-1M and Alpaca, the mean service costs ($2300.7$ and $1808.1$\,ms) are substantially larger than their medians ($1418.7$ and $983.1$\,ms), indicating pronounced right-skewed cost distributions.
ShareGPT exhibits a generally heavier workload, with a mean/median cost of $3560.4/3466.8$\,ms and requests broadly distributed over multi-second execution times.
Notably, LMSYS-Chat-1M and ShareGPT have the same median prompt length of $42$ tokens, yet their median inference costs differ by more than $2\times$, illustrating that similar input-length characteristics do not imply similar dLLM execution workloads.

Overall, these distributions reveal both substantial \textit{within-workload cost heterogeneity} and clear \textit{cross-workload distribution shifts}.
The former creates opportunities for cost-aware request scheduling, while the latter motivates evaluating serving policies across diverse workloads rather than tuning them to a single request distribution.
\section{Predictor Training Details}
\label{app:predictor_training}

We provide additional training details for the coarse-to-fine DWS predictor introduced in Section~\ref{sec:cost_prediction}.
Importantly, the predictor is trained entirely from request execution structure, without using measured latency or any deployment-specific cost factors $\alpha$, $\beta$, or $K(\cdot)$.
This keeps the learned predictor independent of the serving hardware and allows it to be reused with different deployment-specific cost profiles.

\stitle{Predictor input.}
Due to the encoder's maximum input length of 512 tokens, we follow prior work~\citep{ssjf-aiops2024qiu} and left-truncate longer prompts, i.e., only their last 512 tokens are fed to the encoder.
The predictor takes the prompt text and a prompt-length feature as input.

\stitle{Stage A: workload-oriented encoder adaptation.}
Stage A adapts the pretrained MiniLM encoder from generic semantic representation to dLLM denoising workload.
Recall that, for a horizon $h$, the supervision target is
$T^{(h)}=\sum_{b=1}^{n}\min(s_b,h)=\sum_{b=1}^{N_{\max}}\sum_{s=1}^{h}A_{b,s}$.
We use the horizon set $\mathcal{H}=\{1,4,8,12,\ldots,B\}$ and let $\hat T^{(h)}$ denote the corresponding prediction from the Stage-A workload head.
Because the numerical ranges of $T^{(h)}$ increase with $h$, we normalize the prediction and target separately for each horizon before computing the loss.
Denoting the normalization scale for horizon $h$ by $\nu_h$, we optimize
\begin{equation}
\mathcal{L}_{\mathrm{A}}
=
\frac{1}{|\mathcal{H}|}
\sum_{h\in\mathcal{H}}
\operatorname{SmoothL1}
\left(
\frac{\hat T^{(h)}-T^{(h)}}{\nu_h}
\right),
\label{eq:stage_a_loss}
\end{equation}
where
$\operatorname{SmoothL1}(u)=\frac{1}{2}u^2$ for $|u|<1$ and
$\operatorname{SmoothL1}(u)=|u|-\frac{1}{2}$ otherwise.
This horizon-wise normalization prevents large-$h$ targets from dominating optimization merely because of their larger scale.
We do not introduce additional ranking or count losses in this stage.

Training proceeds progressively to avoid abruptly perturbing the pretrained representation.
We first freeze the encoder for one epoch and train only the workload head, and then jointly optimize the head and all 12 encoder layers for another 20 epochs.
The head uses a learning rate of $3\times10^{-4}$, whereas the encoder uses a smaller learning rate of $4\times10^{-5}$.
For both, the learning rate is held constant for the first half of each training phase and then linearly decayed to zero without warm-up.
We use a batch size of 128 and select the Stage-A checkpoint according to validation workload MAE.

For consecutive horizons $h' < h$ in $\mathcal{H}$, the difference $T^{(h)} - T^{(h')} = \sum_{s=h'+1}^{h}\sum_{b} A_{b,s}$ aggregates the number of blocks reaching each step $s \in (h', h]$.
The multi-horizon targets therefore provide binned column sums of DWS in expectation, capturing at a coarse granularity the within-block cost heterogeneity that dominates dLLM inference cost (Appendix~\ref{app:cost-heterogeneity}), while leaving the block dimension and step-level resolution to Stage B.
This is consistent with the ablation in Table~\ref{tab:predictor_eval}, where replacing them with the single total-step target increases cost MAE by $1.25\times$.

\stitle{Stage B: DWS refinement.}
Stage B replaces the coarse multi-horizon supervision with direct supervision on the final DWS.
For each request, its realized activity matrix $A\in\{0,1\}^{N_{\max}\times B}$ provides a binary observation for every denoising cell, while the factorized predictor in Equation~\ref{eq:dws_predictor} produces the corresponding activity probabilities $\hat D\in[0,1]^{N_{\max}\times B}$.
We directly optimize the Brier loss over the complete surface:
\begin{equation}
\mathcal{L}_{\mathrm{B}}
=
\frac{1}{N_{\max}B}
\sum_{b=1}^{N_{\max}}
\sum_{s=1}^{B}
\left(
\hat D_{b,s}-A_{b,s}
\right)^2.
\label{eq:stage_b_loss}
\end{equation}
This objective directly aligns training with the probabilistic interpretation of DWS, while the factorized parameterization itself guarantees the survival structure discussed in Section~\ref{sec:cost_prediction}.
Notably, Stage B still contains no latency-based objective; deployment-specific cell costs are only applied after prediction through the profiled cost model.

We again use progressive unfreezing.
The encoder is frozen for the first four epochs, allowing the newly initialized DWS heads to adapt to the representation learned in Stage A.
We then unfreeze only the top four encoder layers and jointly optimize them with the prediction heads for another four epochs.
The encoder and prediction heads use learning rates of $5\times10^{-6}$ and $2\times10^{-4}$, respectively.
A single learning-rate schedule is maintained continuously across the eight epochs, with a $3\%$ warm-up followed by cosine decay to $5\%$ of the initial learning rate.
The batch size is 256.

\section{Additional Serving Results}
\label{app:additional_online}

\begin{table}[t]
\centering
\caption{\textbf{Additional online serving metrics at RPS=7} (the maximum RPS evaluated in the main-text online experiments).
MiniLM-based PARS-Cost (INT8) uses the \textit{same} encoder and CPU resources as our predictor.
We report average and P90 E2E latency, TTFT, and per-token latency.
Serving model is LLaDA2.0-mini.
Lower is better.}
\label{tab:additional_online_metrics}
\resizebox{\columnwidth}{!}{
\begin{tabular}{lllccc}
\toprule
\textbf{Dataset}
& \textbf{Scheduler}
& \textbf{Statistic}
& \textbf{E2E Latency (s)}
& \textbf{TTFT (s)}
& \textbf{Per-token Latency (s/token)} \\
\midrule

\multirow{6}{*}{\makecell{LMSYS-\\Chat-1M}}
& \multirow{2}{*}{FCFS}
& Avg & 287.93 & 258.57 & 5.52 \\
& & P90 & 674.90 & 646.80 & 15.01 \\
\cmidrule(lr){2-6}
& \multirow{2}{*}{\makecell[l]{MiniLM-based\\PARS-Cost (INT8)}}
& Avg & 237.12 & 205.28 & 2.73 \\
& & P90 & 608.73 & 568.39 & 3.54 \\
\cmidrule(lr){2-6}
& \multirow{2}{*}{\makecell[l]{DWS (INT8)-SJF\\(Ours)}}
& Avg & 149.71 & 116.44 & 0.58 \\
& & P90 & 605.08 & 556.41 & 1.07 \\

\midrule

\multirow{6}{*}{ShareGPT}
& \multirow{2}{*}{FCFS}
& Avg & 543.78 & 507.98 & 6.02 \\
& & P90 & 974.57 & 936.09 & 11.87 \\
\cmidrule(lr){2-6}
& \multirow{2}{*}{\makecell[l]{MiniLM-based\\PARS-Cost (INT8)}}
& Avg & 449.27 & 412.05 & 3.50 \\
& & P90 & 843.76 & 780.19 & 5.00 \\
\cmidrule(lr){2-6}
& \multirow{2}{*}{\makecell[l]{DWS (INT8)-SJF\\(Ours)}}
& Avg & 304.90 & 270.99 & 0.92 \\
& & P90 & 839.08 & 765.32 & 1.49 \\

\midrule

\multirow{6}{*}{Alpaca}
& \multirow{2}{*}{FCFS}
& Avg & 124.76 & 107.77 & 2.91 \\
& & P90 & 217.80 & 201.10 & 9.11 \\
\cmidrule(lr){2-6}
& \multirow{2}{*}{\makecell[l]{MiniLM-based\\PARS-Cost (INT8)}}
& Avg & 97.30 & 82.55 & 1.43\\
& & P90 & 172.79 & 147.84 & 2.40 \\
\cmidrule(lr){2-6}
& \multirow{2}{*}{\makecell[l]{DWS (INT8)-SJF\\(Ours)}}
& Avg & 61.00 & 44.17 & 0.29 \\
& & P90 & 116.80 & 72.47 & 0.66 \\

\bottomrule
\end{tabular}
}
\end{table}

\stitle{Additional metrics.}
We further report normalized latency (i.e., per-token latency), P90 (i.e., 90th percentile) latency and P90 TTFT at RPS=7, the heaviest load in Section~\ref{subsec:application_evaluation}, and additionally include the Alpaca dataset.

As shown in Table~\ref{tab:additional_online_metrics}, DWS (INT8)-SJF consistently reduces all metrics over FCFS across three datasets.
Compared with the improvement of average latency and TTFT, the gains in P90 E2E latency and TTFT are more modest, since the latency tail under SJF mainly consists of expensive requests that are deprioritized by design.
Nevertheless, these tail metrics are still reduced, suggesting that, together with the aging mechanism, SJF does not improve average latency at the cost of starving expensive requests.

\stitle{MiniLM-based PARS-Cost.}
We additionally construct a MiniLM-based PARS-Cost baseline that replaces its original BERT-base encoder with the same MiniLM encoder as DWS, applies INT8 quantization, and runs on the same single CPU core.
Compared with the original BERT-based PARS-Cost, replacing its BERT encoder (110M backbone) with the smaller MiniLM (33M backbone) generally weakens its scheduling performance, likely because the smaller encoder is \textit{less capable} of faithfully representing prompt semantics, making accurate cost prediction more difficult.
Moreover, DWS (INT8)-SJF still outperforms MiniLM-based PARS-Cost across all metrics across three datasets.
These results show that the gains of DWS do not merely come from a lightweight predictor, but from explicitly modeling and predicting the structured denoising cost.

\begin{table}[t]
\centering
\caption{\textbf{Comparison with Oracle-SJF on SDG tasks.}
Oracle-SJF schedules requests by their actual inference costs measured before, serving as an approximate \textit{upper bound} for SJF scheduling.
All predictors are trained on LMSYS-Chat-1M traces collected with LLaDA2.0-mini.}
\label{tab:sdg_performance_with_oracle}
\resizebox{\columnwidth}{!}{
\begin{NiceTabular}{llcccccc}
\toprule
\multirow{2}{*}[-2ex]{\textbf{Serving Model}} & \multirow{2}{*}[-2ex]{\textbf{Dataset}} & \multicolumn{3}{c}{\textbf{Time to 1K Req. (s)} $\downarrow$} & \multicolumn{3}{c}{\textbf{Completed in 3 min} $\uparrow$} \\
\cmidrule(lr){3-5} \cmidrule(lr){6-8}
& & \textbf{FCFS} & \textbf{\makecell{DWS-SJF\\(CPU) (Ours)}} & \textbf{\makecell{Oracle\\-SJF}} & \textbf{FCFS} & \textbf{\makecell{DWS-SJF\\(CPU) (Ours)}} & \textbf{\makecell{Oracle\\-SJF}} \\
\midrule
\multirow{2}{*}{LLaDA2.0-mini} & LMSYS-Chat-1M & 328 & 120 & 107 & 517 & 1299 & 1373\\
& Alpaca & 197 & 75 & 62 & 912 & 1836 & 1959 \\
\midrule
\multirow{2}{*}{LLaDA2.0-flash} & LMSYS-Chat-1M & 1598 & 592 & 531 & 103 & 310 & 350\\
& Alpaca & 909 & 407 & 365 & 198 & 468 & 509\\
\bottomrule
\end{NiceTabular}
}
\end{table}

\stitle{Comparison with Oracle-SJF.}
We further compare DWS-SJF with an Oracle-SJF scheduler on the offline SDG tasks.
Oracle-SJF follows the same SJF scheduling policy as DWS-SJF, but ranks requests by their \textit{actual} inference costs measured before instead of predicted costs.
Since such costs are only available after execution, Oracle-SJF is not deployable in practice and serves as an approximate \textit{upper bound} for prediction-based SJF scheduling.

As shown in Table~\ref{tab:sdg_performance_with_oracle}, DWS-SJF closely approaches Oracle-SJF across all settings, recovering 83.8\%--94.3\% of the oracle improvement over FCFS.
This holds even on the unseen Alpaca dataset and LLaDA2.0-flash model, although the predictor is trained only on LMSYS-Chat-1M traces collected with LLaDA2.0-mini.
In particular, the recovered fraction on LLaDA2.0-flash (83.8\%--94.3\%) is comparable to that on LLaDA2.0-mini (88.3\%--94.1\%), indicating that the DWS shift within the model family has limited impact on scheduling.

The remaining gap between DWS-SJF and Oracle-SJF stems from both prediction error and prediction latency, as Oracle-SJF knows the exact cost of every request upon arrival.
Hence, the DWS predictor already captures \textit{most} of the scheduling benefit attainable with exact cost knowledge, indicating that its prompt-only cost estimates are sufficiently accurate to guide SJF scheduling.

\begin{table}[t]
\centering
\caption{\textbf{Effect of predictor encoder on DWS prediction.}
All variants use the \textit{same} predictor architecture and training setup, differing only in the encoder, and are evaluated without quantization on a single CPU core.
Values in parentheses are relative to MiniLM-L12 (\textit{Ours});
\textcolor{green!50!black}{green} and \textcolor{red!65!black}{red} denote favorable and unfavorable changes, respectively.}
\label{tab:encoder_effect}
\resizebox{\columnwidth}{!}{%
\begin{NiceTabular}{lccc}
\toprule
\textbf{Encoder}
& \textbf{Backbone Params. (M)} $\downarrow$
& \textbf{Cost MAE (ms)} $\downarrow$
& \textbf{Throughput on CPU (req/s)} $\uparrow$ \\
\midrule

MiniLM-L6
& 23 \textcolor{green!50!black}{($0.70\times$)}
& 452.60 \textcolor{red!65!black}{($1.27\times$)}
& 70.72 \textcolor{green!50!black}{($1.95\times$)} \\

MiniLM-L12 (\textit{Ours})
& 33 \textcolor{gray}{($1.00\times$)}
& 355.24 \textcolor{gray}{($1.00\times$)}
& 36.35 \textcolor{gray}{($1.00\times$)} \\

BERT-base-uncased
& 110 \textcolor{red!65!black}{($3.33\times$)}
& 338.09 \textcolor{green!50!black}{($0.95\times$)}
& 10.85 \textcolor{red!65!black}{($0.30\times$)} \\

Qwen3-0.6B~\citep{qwen3technicalreport}
& 600 \textcolor{red!65!black}{($18.18\times$)}
& 324.71 \textcolor{green!50!black}{($0.91\times$)}
& 2.20 \textcolor{red!65!black}{($0.06\times$)} \\

\bottomrule
\end{NiceTabular}%
}
\end{table}

\section{Effect of Predictor Encoder}
\label{app:encoder_effect}

We further study how the encoder capacity affects DWS prediction accuracy and CPU efficiency.
Our default predictor uses MiniLM-L12 (33M backbone), aiming to retain sufficient prompt representation capacity while remaining lightweight for CPU deployment.
We replace it with MiniLM-L6, BERT-base-uncased, and Qwen3-0.6B, while keeping the remaining predictor architecture and training setup unchanged.

\stitle{Results.}
As shown in Table~\ref{tab:encoder_effect}, increasing the encoder size generally improves DWS prediction accuracy, suggesting that a higher-capacity encoder can better represent the prompt semantics relevant to the subsequent denoising behavior.
Accordingly, replacing MiniLM-L12 with the smaller MiniLM-L6 increases cost MAE by $27\%$, although its CPU throughput improves by $1.95\times$.
Conversely, increasing the backbone from MiniLM-L12 (33M) to BERT-base (110M) and Qwen3-0.6B (600M) further reduces cost MAE by $5\%$ and $9\%$, respectively, consistent with their stronger representation capacity.
However, these accuracy gains come at a substantial efficiency cost: their CPU throughput decreases to only $0.30\times$ and $0.06\times$ that of MiniLM-L12.
These results indicate diminishing returns from scaling the encoder: larger backbones provide increasingly stronger prompt representations and modestly improve prediction accuracy, but incur rapidly growing CPU overhead.
We therefore adopt MiniLM-L12 as the default encoder, which retains most of the accuracy benefit of larger encoders while remaining efficient enough for CPU-side prediction.
\section{Cross-Deployment Experiment}
\label{app:cross_deployment}

DWS explicitly decouples the request-dependent execution structure predicted from the prompt from deployment-specific execution costs.
We therefore examine whether the same DWS predictor can be directly reused under a substantially different deployment by re-profiling only the deployment-specific cost factors.

\stitle{Settings.}
We use LLaDA2.0-mini in both deployments to isolate deployment transfer from model transfer.
As summarized in Table~\ref{tab:cross_deployment_config}, the source and target deployments differ substantially in GPU architecture, tensor-parallel configuration, CPU platform, memory hierarchy, interconnect, and software stack.
In particular, the source deployment uses an RTX PRO 6000 Blackwell GPU with TP=1, whereas the target deployment uses RTX A6000 Ampere GPUs with TP=2.

\begin{table}[t]
\centering
\caption{\textbf{Source and target deployment configurations.}
The two deployments differ substantially in both hardware and system stack.
The same trained DWS predictor is reused \textit{without any retraining or fine-tuning}; only the deployment-specific cost factors are re-profiled on the target deployment.}
\label{tab:cross_deployment_config}
\begin{tabularx}{\columnwidth}{lXX}
\toprule
\textbf{Configuration}
& \textbf{Source deployment}
& \textbf{Target deployment} \\
\midrule
GPU
& RTX PRO 6000
& RTX A6000 \\
GPU architecture
& Blackwell
& Ampere \\
GPU memory
& 96 GiB
& 48 GiB \\
GPU power limit
& 600 W
& 300 W \\
Tensor parallelism
& TP=1
& TP=2 \\
PCIe
& PCIe 5.0 $\times$16
& PCIe 4.0 $\times$16 \\
GPU interconnect
& PCIe
& Paired NVLink + PCIe \\
\midrule
CPU
& $2\times$ AMD EPYC 9534
& $2\times$ Intel Xeon Platinum 8375C \\
Physical CPU cores
& 128
& 64 \\
CPU L3 cache
& 512 MiB
& 108 MiB \\
AVX-512 BF16
& Supported
& Not reported \\
Host memory
& 256 GB
& 512 GB \\
\midrule
OS
& Ubuntu 24.04 LTS
& Ubuntu 22.04.5 LTS \\
NVIDIA driver
& 580.126.20
& 550.120 \\
Linux kernel
& 7.0.0-31
& 6.5.0-25 \\
\bottomrule
\end{tabularx}
\end{table}

Importantly, we reuse the identical trained DWS predictor from the source deployment, with \textit{no retraining or fine-tuning} on the target deployment.
We only re-profile $\alpha$, $\beta$, and $K(\cdot)$ on the target deployment following Appendix~\ref{app:profiling}.

We evaluate both online chatbot serving and offline synthetic data generation on the target deployment.
For online serving, we vary RPS from 2 to 8 on LMSYS-Chat-1M and ShareGPT, and report the heaviest load (RPS=8) in Table~\ref{tab:cross_deployment_results}.
For offline SDG, we follow the main-text setting by submitting 3K requests on LMSYS-Chat-1M and Alpaca, and report the time to complete the first 1K requests and the number completed within 3 minutes.

\begin{table}[t]
\centering
\caption{\textbf{Serving performance after cross-deployment transfer.}
The DWS predictor is transferred unchanged from the source deployment, while only the deployment-specific cost profile is refreshed.
Online results are reported at RPS=8, the heaviest load evaluated on the target deployment.
Offline experiments submit 3K requests following the main-text protocol.}
\label{tab:cross_deployment_results}
\resizebox{\columnwidth}{!}{%
\begin{tabular}{llccc}
\toprule
\textbf{Workload}
& \textbf{Metric}
& \textbf{FCFS}
& \textbf{DWS (INT8)-SJF}
& \textbf{Improvement} \\
\midrule
\multirow{2}{*}{Online / LMSYS-Chat-1M}
& Avg. E2E latency (s) $\downarrow$
& 1350.14 & 827.82 & $1.63\times$ \\
& Avg. TTFT (s) $\downarrow$
& 1288.99 & 767.36 & $1.68\times$ \\
\midrule
\multirow{2}{*}{Online / ShareGPT}
& Avg. E2E latency (s) $\downarrow$
& 1679.76 & 1105.02 & $1.52\times$ \\
& Avg. TTFT (s) $\downarrow$
& 1606.41 & 1033.89 & $1.55\times$ \\
\midrule
\multirow{2}{*}{Offline / LMSYS-Chat-1M}
& Time to 1K Req. (s) $\downarrow$
& 678.26 & 250.73 & $2.71\times$ \\
& Completed in 3 min $\uparrow$
& 229 & 747 & $3.26\times$ \\
\midrule
\multirow{2}{*}{Offline / Alpaca}
& Time to 1K Req. (s) $\downarrow$
& 422.22 & 147.95 & $2.85\times$ \\
& Completed in 3 min $\uparrow$
& 443 & 1147 & $2.59\times$ \\
\bottomrule
\end{tabular}
}
\end{table}

\stitle{Results.}
As shown in Table~\ref{tab:cross_deployment_results}, DWS-SJF consistently improves both online and offline serving performance after transferring to the substantially different target deployment.
These gains persist across different workloads and system loads, despite substantial changes in GPU architecture, tensor parallelism, and CPU platform.
Importantly, the DWS predictor is reused \textit{unchanged} throughout the transfer, and only the deployment-specific cost profile is re-profiled.

These results validate the decoupled design of DWS: the predictor captures request-dependent execution structure that remains reusable across deployments, while hardware-specific execution costs are handled separately through lightweight profiling.
Consequently, adapting DWS to a new deployment does not require collecting new training data or retraining the predictor.

\section{Theoretical Analysis of DWS}
\label{app:theory}

This section provides the formal analysis behind the design of DWS.
Appendix~\ref{app:proxy_bias} quantifies the bias of the two scalar workload proxies and illustrates that of the direct-cost target with an example.
Appendix~\ref{app:theory_proof} proves Proposition~\ref{pro:error}, which bounds the cost-estimation error of any DWS predictor under every deployment.
Appendix~\ref{app:factorized} then shows that the factorized parameterization in Equation~\ref{eq:dws_predictor} is expressively complete.

\stitle{Setup.}
To formalize transfer across hardware deployments, we fix the serving model, block size, and decoding configuration, so that the execution distribution $A \mid x$ is shared by all deployments considered, and deployments differ only in their profiled cost factors $\theta = (\alpha, \beta, K(\cdot))$.
The limitations in Section~\ref{sec:discussion} further allow changing the serving model within the same model family, as evaluated in Section~\ref{subsec:application_evaluation}; we discuss this case in the remarks of Appendix~\ref{app:theory_proof}.

We write $\langle \cdot, \cdot \rangle$ for the Frobenius inner product, $T = \sum_{b,s} A_{b,s}$ for the total number of denoising steps, and $U$ for the matrix that equals one in its first column and zero elsewhere, so that $n = \langle U, A \rangle$.

Since each generated block triggers exactly one KV refresh, $\beta$ can be absorbed into the first column of the cell-cost matrix.
With $\widetilde K_\theta(m)\triangleq K(m)+\beta U$, Equation~\ref{eq:cost_final} becomes $C_\theta=\alpha m+\langle\widetilde K_\theta(m),A\rangle$, which is linear in $A$.
Since $m(x)$ is determined by the prompt and $\mathbb{E}[A\mid x]=D(x)$, the linearity of conditional expectation gives
\begin{equation}
\mathbb{E}[C_\theta\mid x]=\alpha m(x)+\langle\widetilde K_\theta(m(x)),D(x)\rangle,
\label{eq:expected_linear}
\end{equation}
which is Equation~\ref{eq:expected_cost} in matrix form and underlies all analyses below.

\subsection{Bias of Scalar Prediction Targets}
\label{app:proxy_bias}

Section~\ref{sec:intro} shows empirically that output length and total denoising steps are inaccurate proxies for dLLM workload, and Section~\ref{sec:discussion} attributes this to their lossy compression of DWS.
Here we quantify their bias under a given deployment, and show that no fixed calibration constant makes either proxy unbiased.
We then illustrate the bias of the direct-cost target under deployment changes with an example.
For the two proxies, taking the conditional expectation of Equation~\ref{eq:denoising_final} gives $\mathbb{E}[C_{\mathrm{denoising}}\mid x]=\sum_{b,s}K_{b,s}(m(x))\,D_{b,s}(x)$.

\stitle{Total denoising steps.}
A total-step proxy assigns the same cost $\bar c$ to every denoising step, regardless of its position or the prompt length.

\begin{proposition}[Bias of the total-step proxy]
\label{prop:bias_steps}
For any constant $\bar c$ and any prompt $x$,
\begin{equation}
\mathbb{E}[C_{\mathrm{denoising}}\mid x]-\bar c\,\mathbb{E}[T\mid x]
=\sum_{b,s}\big(K_{b,s}(m(x))-\bar c\big)D_{b,s}(x).
\label{eq:bias_steps}
\end{equation}
\end{proposition}

\begin{proof}
The claim follows from $\mathbb{E}[T\mid x]=\sum_{b,s}D_{b,s}(x)$.
\end{proof}

Thus, the proxy is unbiased for $x$ if and only if $\bar c$ equals the effective mean step cost $\mu_T(x)\triangleq\sum_{b,s}K_{b,s}(m(x))D_{b,s}(x)/\sum_{b,s}D_{b,s}(x)$.
A single constant $\bar c$ is therefore unbiased for all prompts only if $\mu_T(x)$ is invariant across prompts, which fails in general because both $D(x)$ and $K(m(x))$ vary across requests.

\stitle{Output length.}
The output-length proxy further assumes that every generated block requires the same number of denoising steps $\bar s$, and estimates the denoising cost as $\bar c\bar s\,\widetilde L/B$, where $\widetilde L\triangleq Bn$ is the block-aligned output length.

\begin{proposition}[Bias of the output-length proxy]
\label{prop:bias_length}
For any constants $\bar c$ and $\bar s$ and any prompt $x$,
\begin{equation}
\mathbb{E}[C_{\mathrm{denoising}}\mid x]-\frac{\bar c\bar s}{B}\,\mathbb{E}[\widetilde L\mid x]
=\underbrace{\sum_{b,s}\big(K_{b,s}(m(x))-\bar c\big)D_{b,s}(x)}_{\text{cell-cost heterogeneity}}
+\underbrace{\bar c\,\big(\mathbb{E}[T\mid x]-\bar s\,\mathbb{E}[n\mid x]\big)}_{\text{step-count heterogeneity}}.
\label{eq:bias_length}
\end{equation}
\end{proposition}

\begin{proof}
The claim follows from $\mathbb{E}[\widetilde L\mid x]/B=\mathbb{E}[n\mid x]$ by adding and subtracting $\bar c\,\mathbb{E}[T\mid x]$.
\end{proof}

The first term is the bias of the total-step proxy, while the second term is specific to output length and vanishes if and only if $\bar s=\mathbb{E}[T\mid x]/\mathbb{E}[n\mid x]$, i.e., the expected number of denoising steps per generated block.
For the actual output length $L=\widetilde L-R$ with $0\le R<B$, the bias gains an additional term $\bar c\bar s\,\mathbb{E}[R\mid x]/B$, which is at most the proxy cost of a single block.

\stitle{A worked example.}
Let $B=4$ and let the cell costs be $K_{b,\cdot}=(5,6,8,10)$\,ms for every block $b$.
Trace $\tau_1$ generates four blocks with one denoising step each, whereas trace $\tau_2$ generates one block with four steps, so both execute $T=4$ steps.
Their denoising costs are $20$\,ms and $29$\,ms, respectively, yet the total-step proxy assigns them the same cost.
The output-length proxy is worse: since $\tau_1$ generates four times as many blocks, it predicts $\tau_1$ to be four times as costly, reversing the true order and thus misguiding SJF scheduling.

\stitle{Direct cost.}
A direct-cost predictor regresses the measured cost under its training deployment $\theta_0=(\alpha_0,\beta_0,K_0(\cdot))$, so its ideal prediction is $\mathbb{E}[C_{\theta_0}\mid x]$.
Reusing this predictor on a new deployment $\theta$ without retraining requires some form of calibration.
As an example, we consider the simplest calibration, which rescales the prediction by a constant $\lambda>0$.

\begin{proposition}[Bias of the direct-cost target under constant rescaling]
\label{prop:bias_direct}
For any deployment $\theta$, any constant $\lambda>0$, and any prompt $x$,
\begin{equation}
\mathbb{E}[C_\theta\mid x]-\lambda\,\mathbb{E}[C_{\theta_0}\mid x]
=(\alpha-\lambda\alpha_0)\,m(x)+\big\langle\widetilde K_\theta(m(x))-\lambda\widetilde K_{\theta_0}(m(x)),\,D(x)\big\rangle,
\label{eq:bias_direct}
\end{equation}
where $\widetilde K_{\theta_0}(m)\triangleq K_0(m)+\beta_0U$.
\end{proposition}

\begin{proof}
The claim follows by applying Equation~\ref{eq:expected_linear} to both $\theta$ and $\theta_0$.
\end{proof}

In this example, the bias vanishes if $\theta$ uniformly rescales $\theta_0$, i.e., $\alpha=\lambda\alpha_0$ and $\widetilde K_\theta=\lambda\widetilde K_{\theta_0}$.
Otherwise, the per-cell cost differences are weighted by $D(x)$, which varies across prompts, so no single $\lambda$ removes the bias in general.
In contrast, DWS only requires re-profiling $\theta$ (Appendix~\ref{app:profiling}), after which Equation~\ref{eq:expected_linear} applies directly without retraining the predictor.

\subsection{Proof of the error-control bound}
\label{app:theory_proof}

We first formalize the quantities in Proposition~\ref{pro:error}.
A predictor $\widehat D$ is any measurable function of the prompt taking values in $[0,1]^{N_{\max}\times B}$.
Following Equation~\ref{eq:expected_linear}, its predicted cost under deployment $\theta$ is $\widehat C_\theta(x)=\alpha m(x)+\langle\widetilde K_\theta(m(x)),\widehat D(x)\rangle$, the constant in Proposition~\ref{pro:error} is $\kappa_\theta\triangleq\sup_{m}\|\widetilde K_\theta(m)\|_F$, and its Brier risk is
\begin{equation}
\mathcal{R}(\widehat D)\triangleq\frac{1}{N_{\max}B}\,\mathbb{E}_{x,A}\big\|\widehat D(x)-A\big\|_F^2,
\label{eq:brier_risk}
\end{equation}
whose empirical counterpart is the Stage-B loss in Appendix~\ref{app:predictor_training}.
Next, we show that the excess Brier risk equals the mean squared distance between the predicted and true DWS.

\begin{lemma}[Brier decomposition]
\label{lem:brier}
For any predictor $\widehat D$,
\begin{equation}
\mathcal{R}(\widehat D)-\mathcal{R}(D)=\frac{1}{N_{\max}B}\,\mathbb{E}_x\big\|\widehat D(x)-D(x)\big\|_F^2.
\label{eq:brier_excess}
\end{equation}
Hence, $D$ is the unique minimizer of $\mathcal{R}$ up to almost-sure equality.
\end{lemma}

\begin{proof}
Given $x$, expanding $\|\widehat D(x)-A\|_F^2$ around $D(x)$ yields the cross term $2\langle\widehat D(x)-D(x),D(x)-A\rangle$, which vanishes under $\mathbb{E}_{A\mid x}$ since $\mathbb{E}[A\mid x]=D(x)$.
Thus, $\mathbb{E}_{A\mid x}\|\widehat D(x)-A\|_F^2=\|\widehat D(x)-D(x)\|_F^2+\mathbb{E}_{A\mid x}\|D(x)-A\|_F^2$, and taking the expectation over $x$ proves the identity.
Its right-hand side is nonnegative and vanishes if and only if $\widehat D=D$ almost surely.
\end{proof}

\begin{proof}[Proof of Proposition~\ref{pro:error}]
By Equation~\ref{eq:expected_linear} and the Cauchy--Schwarz inequality,
\begin{equation}
\big|\widehat C_\theta(x)-\mathbb{E}[C_\theta\mid x]\big|
=\big|\langle\widetilde K_\theta(m(x)),\widehat D(x)-D(x)\rangle\big|
\le\kappa_\theta\big\|\widehat D(x)-D(x)\big\|_F.
\label{eq:pointwise_bound}
\end{equation}
Taking the expectation over $x$, applying Jensen's inequality $\mathbb{E}_x\|\widehat D(x)-D(x)\|_F\le\big(\mathbb{E}_x\|\widehat D(x)-D(x)\|_F^2\big)^{1/2}$, and invoking Lemma~\ref{lem:brier} yields Equation~\ref{eq:uniform_bound}.
\end{proof}

\stitle{Remarks.}
Proposition~\ref{pro:error} assumes a fixed serving model and decoding configuration, as specified in the Setup.
If the execution distribution changes (e.g., under a different serving model) while $B$ and $N_{\max}$ are kept fixed, the target DWS $D'(x)$ may differ from the source DWS $D(x)$ that $\widehat{D}$ is trained to approximate.
By the triangle inequality, Equation~\ref{eq:uniform_bound} then gains an additional \textit{shift} term $\kappa_\theta\,\mathbb{E}_x\|D'(x)-D(x)\|_F$, where $R$ remains the Brier risk on the source model.
The experimental results in Section~\ref{subsec:application_evaluation} and Appendix~\ref{app:additional_online} show that this shift has limited practical impact within the same model family.
When the predictor trained on LLaDA2.0-mini is reused on LLaDA2.0-flash without retraining, DWS-SJF still recovers most of the Oracle-SJF improvement over FCFS, comparable to that on LLaDA2.0-mini.

\subsection{Expressive Completeness of the Factorized DWS}
\label{app:factorized}

Section~\ref{subsec:prediction_target} shows that the factorized parameterization always satisfies the survival structure.
Here we further show that it loses no expressiveness, i.e., it can represent every DWS that an execution can produce.
Let
\begin{equation}
\mathcal{P}\triangleq\Big\{D\in[0,1]^{N_{\max}\times B}:\;D_{1,1}=1,\;D_{b,s+1}\le D_{b,s},\;D_{b+1,1}\le D_{b,1}\Big\}
\label{eq:valid_dws}
\end{equation}
denote the set of matrices satisfying this structure, which contains every DWS since $A_{1,1}=1$ and both monotonicity constraints hold pointwise for $A$.

\begin{proposition}[Expressive completeness]
\label{prop:factorized_complete}
A matrix $D\in[0,1]^{N_{\max}\times B}$ admits the factorization in Equation~\ref{eq:dws_predictor}, i.e.,
\begin{equation}
D_{b,s}=\Big(\sum_{k=b}^{N_{\max}}p_k^{\mathrm{block}}\Big)\Big(\sum_{t=s}^{B}p_{b,t}^{\mathrm{step}}\Big)
\label{eq:factorized_app}
\end{equation}
for some distribution $p^{\mathrm{block}}$ over $\{1,\dots,N_{\max}\}$ and distributions $p_b^{\mathrm{step}}$ over $\{1,\dots,B\}$, if and only if $D\in\mathcal{P}$.
Moreover, $\mathcal{P}$ is exactly the set of DWS matrices induced by some execution distribution.
\end{proposition}

\begin{proof}
Since survival probabilities are nonincreasing and start at one, every matrix of the form in Equation~\ref{eq:factorized_app} lies in $\mathcal{P}$.
Conversely, given $D\in\mathcal{P}$, let $q_b\triangleq D_{b,1}$ with $q_{N_{\max}+1}\triangleq0$, and define $p_b^{\mathrm{block}}\triangleq q_b-q_{b+1}$, which is a valid distribution with survival probabilities $q_b$.
For every block with $q_b>0$, let $D_{b,B+1}\triangleq0$ and define $p_{b,t}^{\mathrm{step}}\triangleq(D_{b,t}-D_{b,t+1})/q_b$, which is a valid distribution with survival probabilities $D_{b,s}/q_b$; if $q_b=0$, then $D_{b,\cdot}=0$ and $p_b^{\mathrm{step}}$ is arbitrary.
Multiplying the two survival probabilities recovers $D_{b,s}$.
Finally, drawing $n\sim p^{\mathrm{block}}$ and, independently, $s_b\sim p_b^{\mathrm{step}}$ for each block yields an execution with $\mathbb{P}(A_{b,s}=1)=D_{b,s}$, so every $D\in\mathcal{P}$ is realizable.
\end{proof}

Since the expected cost depends on $A$ only through its cell-wise activity probabilities (Equation~\ref{eq:expected_linear}), Proposition~\ref{prop:factorized_complete} implies that the factorization entails no loss for cost estimation, even though it does not model the joint distribution of $A$.
\section{Additional Related Work}
\label{app:more_related_work}

\stitle{Diffusion LLMs and inference frameworks.}
D3PM~\citep{austin2021structured} formulates diffusion models over discrete state spaces, and MDLM~\citep{sahoo2024simple} simplifies its absorbing-state variant into masked diffusion language models, which generate text by iteratively denoising masked tokens.
Building on this paradigm, dLLMs have recently been scaled into competitive language models, such as LLaDA~\citep{nie2026large}, Dream~\citep{ye2025dream7bdiffusionlarge}, and LLaDA2.0~\citep{bie2025llada2}.
To accelerate their inference, several studies introduce dLLM-specific KV-cache mechanisms, such as Fast-dLLM~\citep{wu2026fast} and dKV-Cache~\citep{ma2026dkv}, which are \textit{orthogonal} to DWS.
Meanwhile, SGLang and vLLM have begun to natively support dLLMs, and dInfer~\citep{ma2025dinfer} provides a dedicated inference framework for them.
These advances make accurate request-level cost estimation increasingly important for practical dLLM serving.

\stitle{Cost-aware LLM serving.}
For AR LLMs, \textit{output length} provides a natural proxy for inference cost.
Accordingly, a line of work predicts either output length or its relative ordering to approximate shortest-job-first scheduling and mitigate head-of-line blocking~\citep{ssjf-aiops2024qiu,fu2024efficient,tao2026ranking,wang2026robust}.
Predicted output lengths also guide batching and memory management: \citet{zheng2023response} group requests with similar predicted response lengths into the same micro-batch to reduce padding waste, while S$^3$~\citep{jin2023s} allocates KV-cache memory according to predicted output lengths to enlarge the batch size and improve throughput.

Beyond non-preemptive scheduling, FastServe~\citep{wu2026fastserve} and TRAIL~\citep{dont-stop-shahoutdon} preempt running AR requests at token granularity to prioritize shorter ones.
Block-autoregressive dLLMs naturally admit similar preemption at block boundaries, where no partially denoised state needs to be preserved.
With its block-level fine-grained predictions (Section~\ref{sec:discussion}), DWS can readily estimate the remaining workload of in-flight requests to support such preemption.
This makes DWS-guided preemptive scheduling a promising future work for further improving dLLM serving.

\end{document}